%% file: JMLR-2016-v2.tex
\documentclass[twoside,11pt]{article}

\usepackage{jmlr2e}

\usepackage{hyperref}       
\usepackage{url}            
\usepackage{booktabs}       
\usepackage{amsfonts}       

\usepackage{amsmath,amssymb}
\usepackage{comment}

\newcommand{\ibf}{\mathbf{i}}
\newcommand{\jbf}{\mathbf{j}}
\usepackage{algorithm}
\usepackage{algorithmic}
\usepackage{hyperref}
\usepackage{natbib}
\usepackage[caption=false,lofdepth,lotdepth]{subfig}
\usepackage{color}

\jmlrheading{1x}{201x}{x-xx}{x/0x}{x/0x}{H\`a Quang Minh}

\ShortHeadings{Feature Maps for Operator-Valued Kernels}{H.Q.Minh}
\firstpageno{1}

\begin{document}

\input style.tex

\title{Operator-Valued Bochner Theorem, Fourier Feature Maps for Operator-Valued Kernels, and Vector-Valued Learning}

\author{\name  H\`a Quang Minh  \email  minh.haquang@iit.it\\
        \addr Pattern Analysis and Computer Vision (PAVIS)\\
   			 Istituto Italiano di Tecnologia (IIT), Via Morego 30, Genova 16163, ITALY}

\editor{}

\maketitle

\begin{abstract}
This paper presents a framework for computing random operator-valued feature maps for operator-valued positive definite kernels.
This is a generalization of the random Fourier features for scalar-valued kernels to the operator-valued case. 
Our general setting is that of operator-valued kernels corresponding to RKHS of functions with values in a Hilbert space.
We show that in general, for a given kernel, there are potentially infinitely many random feature maps, which can be bounded or unbounded.
Most importantly, given a kernel, we present a general, closed form formula for computing a corresponding probability measure, which is required for the construction of the Fourier features, and which, unlike the scalar case, is not uniquely and automatically determined by the kernel. We also show that, under appropriate conditions, random bounded feature maps can always be computed.
Furthermore, we show the uniform convergence, under the Hilbert-Schmidt norm, of the resulting approximate kernel to the exact kernel 
on any compact subset of Euclidean space. Our convergence requires differentiable kernels, an improvement over the twice-differentiability requirement in previous work in the scalar setting. We then show how operator-valued feature maps and their approximations can be employed in a general vector-valued learning framework.
The mathematical formulation is illustrated by numerical examples on matrix-valued kernels.
\end{abstract}

\section{Introduction}

The current work is concerned with the construction of random feature maps for operator-valued kernels and their applications in vector-valued learning. Much work has been done in machine learning recently on these kernels and their associated RKHS of vector-valued functions, both theoretically and practically,
see e.g. 
\citep{MichelliPontil05, Carmeli2006, Reisert2007, Caponnetto08, ICML2011Brouard, ICML2011Dinuzzo,Kadrietal2011, MinhVikasICML2011, Zhangetal:JMLR2012, VikasMinhLozano:UAI2013}. While rich in theory and potentially powerful in applications, one of the main challenges in applying operator-valued kernels is that they are computationally intensive on large datasets. 
In the scalar setting, one of the most powerful approaches for scaling up kernel methods is Random Fourier Features \citep{Fourier:NIPS2007}, which applies Bochner's Theorem and the Inverse Fourier Transform to build random features that approximate a given shift-invariant kernel. The approach in \citep{Fourier:NIPS2007} has been improved both in terms of  computational speed \citep{Fastfood:ICML2013} 
and rates of convergence \citep{Fourier:UAI2015,Fourier:NIPS2015}. 

{\bf Our contributions}. The following are the contributions of this work.
\begin{enumerate}
\item {\it Firstly}, we construct random feature maps for operator-valued shift-invariant kernels using the operator-valued version of Bochner's Theorem. The {\it key differences} between the operator-valued and scalar settings are the following. The first key difference is that, in the scalar setting, a positive definite function $k$, with normalization, is the Fourier transform of a probability measure $\rho$, which is {\it uniquely determined} as the inverse Fourier transform of $k$. In the operator-valued setting, $k$ is the Fourier transform of a {\it unique} finite positive operator-valued measure $\mu$. However, the probability measure $\rho$, which is necessary
for constructing the random feature maps, must be {\it explicitly constructed}, that is it is {\it not} automatically determined by $k$. 
In this work, we present a general formula for computing a probability measure $\rho$ given a kernel $k$.
The second key difference is that, in the operator-valued setting, the probability measure $\rho$ is generally {\it non-unique}, being a factor of $\mu$. As a consequence, we show
that in general, there are (potentially infinitely) many random feature maps, which may be either unbounded or bounded. However, under appropriate assumptions, we show that there always exist bounded feature maps. This is true for many of the commonly encountered kernels, including separable kernels and curl-free and divergence-free kernels. 

\item {\it Secondly}, for the bounded feature maps, we show that the associated approximate kernel converges uniformly to the exact kernel in Hilbert-Schmidt norm on any compact subset in Euclidean space.

\item {\it Thirdly}, when restricting to the scalar setting, our convergence holds for differentiable kernels, which is an improvement 
over the hypothesis of \citep{Fourier:NIPS2007,Fourier:UAI2015,Fourier:NIPS2015,Romain:2016}, which all require the kernels to be twice-differentiable. 

\item {\it Fourthly}, we show how operator-valued feature maps and their approximations can be used directly in a general learning formulation in RKHS.

\end{enumerate}

{\bf Related work}. The work most closely related to our present work is \citep{Romain:2016}. While the {\it formal constructions} of the Fourier feature maps in \citep{Romain:2016} and our work are similar, there are several {\it crucial differences}. 
The first and most important difference is that in \citep{Romain:2016} there is {no} general mechanism for computing a probability measure $\rho$, 
which is required for the construction of the Fourier feature maps. As such, the results presented in \citep{Romain:2016}
are only for three specific kernels, namely separable kernels, curl-free and div-free kernels, {not} for a general kernel as in our setting.
Moreover, for the curl-free and div-free kernels, \citep{Romain:2016} presented unbounded feature maps, whereas we show
that, apart from unbounded feature maps, there are generally infinitely many bounded feature maps associated with these kernels.
Secondly, more general than the matrix-valued kernel, i.e finite-dimensional, setting in \citep{Romain:2016}, we
work in the operator-valued kernel setting, with RKHS of functions with values in a Hilbert space. In this setting, the convergence
in the Hilbert-Schmidt norm that we present is {\it strictly stronger} than the convergence in spectral norm given in \citep{Romain:2016}.
At the same time, our convergence requires {\it weaker assumptions} than those in \citep{Romain:2016} and previous results in the scalar setting
\citep{Fourier:NIPS2007,Fourier:UAI2015,Fourier:NIPS2015}.

{\bf Organization}. 
We first briefly review random Fourier features and operator-valued kernels in Section~\ref{section:background}.
Feature maps for operator-valued kernels are described in Section~\ref{section:operator-feature}.
The core of the paper is Section~\ref{section:operator-random}, which describes
the construction of random feature maps using operator-valued Bochner's Theorem, the computation of the required probability measure, along with the uniform convergence of the corresponding approximate kernels. Section~\ref{section:learning} employs feature maps and their approximations in a general vector-valued learning formulation, with the accompanying experiments in Section~\ref{section:experiments}.
All mathematical proofs are given in Appendix \ref{section:proofs}.

\section{Background}
\label{section:background}
Throughout the paper, we work with shift-invariant positive definite kernels {$K$} on {$\R^n \times\R^n$},
so that {$K(x,t) = k(x-t) \forall x,t\in\R^n$} for some function {$k:\R^n \mapto \R$}, which is then said to be a {\it positive definite function} on $\R^n$.  

{\bf Random Fourier features for scalar-valued kernels \citep{Fourier:NIPS2007}}. Bochner's Theorem in the scalar setting, see e.g. \citep{ReedSimon:vol2}, states that 
a complex-valued, continuous function $k$ on $\R^n$ is positive definite if and only if it is the Fourier transform of a finite, positive measure $\mu$ on $\R^n$, that is
{
\begin{align}
\label{equation:Bochner-scalar}
k(x) = \hat{\mu}(x) = \int_{\R^n}e^{-i\la \omega, x\ra}d\mu(\omega).
\end{align}
}
For our purposes, we consider exclusively the {\it real-valued} setting for $k$.
Since $\mu$ is a finite positive measure, without loss of generality, we assume that $\mu$ is a probability measure, so that {$k(x) = \bE_{\mu}[e^{-i\la\omega, x\ra}]$}.
The measure $\mu$ is {\it uniquely determined} via $\hat{\mu} = k$.  
For the Gaussian function {$k(x) = e^{-\frac{||x||^2}{\sigma^2}}$}, we have 
{$\mu(\omega) = \frac{(\sigma\sqrt{\pi})^n}{(2\pi)^n}e^{-\frac{\sigma^2||\omega||^2}{4}}
\sim \Ncal\left(0, \frac{2}{\sigma^2}\right)$}.
Consider now the
kernel {$K(x,t) = k(x-t) = \int_{\R^n}e^{-i\la \omega, x-t\ra}d\mu(\omega)$}. Using the symmetry of {$K$} and the relation {$\frac{1}{2}(e^{ix} + e^{-ix}) = \cos(x)$}, we obtain
{
\begin{align}
K(x,t) = \frac{1}{2}\int_{\R^n}[e^{i\la \omega, x-t\ra} + e^{-i \la \omega, x-t\ra}]d\mu(\omega) = \int_{\R^n}\cos(\la \omega, x-t\ra)d\mu(\omega).
\end{align}
}
Let {$\{\omega_j\}_{j=1}^D$} be points in {$\R^n$}, independently sampled according to the measure {$\mu$}. Then we have an empirical approximation {$\hat{K}_D$} of {$K$} and the associated feature map {$\hat{\Phi}_D:\R^n \mapto \R^{2D}$}, as follows
{
\begin{align}
\hat{K}_D(x,t) &= \frac{1}{D}\sum_{j=1}^D\cos(\la \omega_j, x-t\ra) = \frac{1}{D}\sum_{j=1}^D[\cos(\la \omega_j, x\ra)\cos(\la \omega_j, t\ra) + \sin(\la \omega_j,x\ra)\sin(\la \omega_j, t\ra)
\nonumber
\\
& = \la \hat{\Phi}_D(x), \hat{\Phi}_D(t)\ra, \;\;\;\text{where}\;\;\; \hat{\Phi}_D(x) = (\cos(\la \omega_j, x\ra),\sin(\la \omega_j, x\ra))_{j=1}^D \in \R^{2D}.
\label{equation:feature-scalar}
\end{align}
}
The current work generalizes the feature map {$\hat{\Phi}_D$} above to the case {$K$} is an operator-valued kernel and the corresponding $\mu$ is a positive operator-valued measure.
{\bf Vector-valued RKHS}.
Let us now
briefly recall operator-valued kernels and their corresponding RKHS of vector-valued functions,
for more detail see 
e.g.
\citep{Carmeli2006, MichelliPontil05,
Caponnetto08, MinhVikasICML2011}.
Let $\X$ be a nonempty set, $\mathcal{W}$ a real, separable Hilbert space with inner product $\langle \cdot,\cdot\rangle_{\mathcal{W}}$, $\mathcal{L}(\W)$ the Banach space of bounded linear operators on $\W$.
Let
{$\W^{\X}$} denote the vector space of all functions {$f:\X \rightarrow \W$}. A function {$K: \X \times \X \rightarrow
\mathcal{L}(\W)$} is said to be an {\it operator-valued positive definite kernel} if  for each pair {$(x,t) \in
\X \times \X$, $K(x,t)^{*} = K(t,x)$},
and for every set of points {$\{x_i\}_{i=1}^N$} in $\X$ and
{$\{w_i\}_{i=1}^N$} in {$\W$}, {$N \in \N$},
{
$\sum_{i,j=1}^N\langle w_i, K(x_i,x_j)w_j\rangle_\W \geq 0$.
}
%
For {$x\in \X$} and {$w \in \W$}, form a function {$K_xw = K(.,x)w \in \W^{\X}$}
by
{
\begin{align}
(K_xw)(t) = K(t,x) w  \;\;\;\;  \forall t \in \X.
\end{align}
}
Consider the set {$\mathcal{H}_0 = {\rm span}\{K_xw | x \in \X, w \in \W\} \subset \W^\X$}.
For {$f= \sum_{i=1}^NK_{x_i}w_i$, $g = \sum_{i=1}^NK_{z_i}y_i
\in \mathcal{H}_0$}, we define the inner product
{$\langle f, g \rangle_{\H_K} = \sum_{i,j=1}^N\langle w_i, K(x_i,z_j)y_j\rangle_\W$},
which makes $\mathcal{H}_0$ a pre-Hilbert space.
Completing
$\mathcal{H}_0$ by adding the limits of all Cauchy sequences gives the Hilbert space $\mathcal{H}_K$. This is the reproducing kernel Hilbert space (RKHS) of {$\W$}-valued functions on {$\X$}. The {\it reproducing property} is
\begin{align}
\label{equation:reproducing2}
\langle f(x),y\rangle_\W = \langle f, K_xy\rangle_{\H_K} \;\;\;\; \mbox{for all} \;\;\; f \in \mathcal{H}_K.
\end{align}

\subsection{Operator-Valued Feature Maps for Operator-Valued Kernels}
\label{section:operator-feature}
Feature maps for 
operator-valued kernels were first considered in \citep{Caponnetto08}.
%
Let {$\F_K$} be a separable Hilbert space and {$\mathcal{L}(\W, \F_K)$} be
the Banach space of all bounded linear operators mapping from {$\W$ to $\F_K$}. A {\it feature map} for {$K$} with corresponding {\it feature space} {$\F_K$} is a mapping
{
\begin{align}
\label{equation:feature-def}
\Phi_K: \X \mapto \mathcal{L}(\W, \F_K), \;\;\;\text{such that}\;\;\;K(x,t) = \Phi_K(x)^{*}\Phi_K(t) \;\;\; \forall (x,t) \in \X \times \X.
\end{align}
}
The operator-valued map {$\Phi_K$} is generally nonlinear as a function on $\X$. For each {$x \in \X$}, 
{$\Phi_K(x) \in \mathcal{L}(\W, \F_K)$} and 
\begin{align}
\la w, K(x,t)w\ra_{\W} = \la w, \Phi_K(x)^{*}\Phi_K(t)\ra_{\W}
= \la \Phi_K(x)w, \Phi_K(t)w\ra_{\F_K}.
\end{align}
In the following, for brevity, we also refer to the pair {$(\Phi_K, \F_K)$} as a feature map for {$K$}. 

{\bf Existence of operator-valued feature maps and the canonical feature map}.
Let $K$ be any operator-valued positive definite kernel on $\X \times \X$, we now show that then there always exists at least one feature map, as follows.
For each $x \in \X$, consider the linear operator $K_x: \W \mapto \H_K$ defined by
$K_xw(t) = K(t,x)w$, $x,t \in \X$, as above.
Then
{
\begin{align}
||K_xw||_{\H_K}^2 = \langle K(x,x)w, w\rangle_{\W} \leq ||K(x,x)||\;||w||^2_{\W},
\end{align}
}
which implies that $K_x$ is a bounded operator, with
{
\begin{align}
||K_x: \W \rightarrow \H_K|| \leq \sqrt{||K(x,x)||},
\end{align}
}
Let $K_x^{*}: \H_K \mapto \W$ be the adjoint operator for $K_x$.
The reproducing property states that $\forall w \in \W$,
{
\begin{equation}
\la f(x), w\ra_{\W} = \la f, K_xw\ra_{\H_K} = \la K_x^{*}f, w\ra_{\W} \imply K_x^{*}f = f(x).
\end{equation}
}
For any $u,v \in \W$, we have 
{
\begin{align}
\la u, K(x,t)v\ra_{\W} = \la u, K_tv(x)\ra_{\W} = \la u, K_x^{*}K_tv\ra_{\W}
= \la K_xu, K_tv\ra_{\H_K} \imply K(x,t) = K_x^{*}K_t,
\end{align}
}
from which it follows that
{
\begin{align}
\Phi_K: \X \mapto \mathcal{L}(\W,\H_K), \;\;\;
\Phi_K(x) = K_x \in \mathcal{L}(\W,\H_K)
\end{align}
}
is a feature map for $K$ with feature space $\H_K$, which exists for any positive definite kernel $K$. Following the terminology in the scalar setting \citep{Minh-Niyogi-Yao}, we also call it
the {\it canonical feature map} for $K$.

\begin{remark} In \citep{Caponnetto08}, it is {\it assumed} that the kernel has the representation {$K(x,t) = \Phi_K(x)^{*}\Phi_K(t)$}. However, as we have just shown,
for any positive definite kernel {$K$}, there is always at least one such representation, given by the canonical feature map above.
\end{remark}

Similar to the scalar setting \citep{Minh-Niyogi-Yao}, feature maps are generally {\it non-unique}, as we show below. However, they are all essentially equivalent, similar to the scalar case, as shown by the following.

\begin{lemma}\label{lemma:f-feature} 
Let {$(\Phi_K, \F_K)$} be any feature map for {$K$}. Then {$\forall f \in \H_K$}, there exists an {$\h \in \F_K$} such that
\begin{equation}
f(x) = K_x^{*}f = \Phi_K(x)^{*}\h,\;\;\;\forall x \in \X. 
\end{equation}
Furthermore,
{
$||f||_{\H_K} = ||\h||_{\F_K}$.
}
\end{lemma}




%
%







\section{Random Operator-Valued Feature Maps}
\label{section:operator-random}

We now present the generalization of the random Fourier feature map from the scalar setting to the operator-valued setting.
We begin by reviewing Bochner's Theorem in the operator-valued setting in Section \ref{section:bochner}, which immediately leads to the formal construction of the Fourier feature maps in Section \ref{section:feature-construction}.
As we stated, in the operator-valued setting, we need to explicitly construct the required probability measure. This is done individually for some specific kernels in Section \ref{section:special-construction} and for a general kernel in Section \ref{section:probability-construction}.

\subsection{Operator-Valued Bochner Theorem}
\label{section:bochner}
The operator-valued version of Bochner's Theorem that we present here is from \citep{Neeb:Operator1998},
see also \citep{Falb:1969,Carmeli:2010}.
Throughout this section, let {$\H$} be a separable Hilbert space. Let {$\L(\H)$} denote the Banach space of bounded linear operators on {$\H$},
{$\Sym(\H) \subset \L(\H)$} denote the subspace of bounded, self-adjoint operators on {$\H$}, and
{$\Sym^{+}(\H) \subset \Sym(\H)$} denote the set of self-adjoint, bounded, positive operators on {$\H$}. An operator {$A \in \L(\H)$} is said to be trace class, denoted by {$A \in \Tr(\H)$}, if {$\sum_{k=1}^{\infty}\la\e_k, (A^{*}A)^{1/2}\e_k\ra < \infty$} for 
any orthonormal basis {$\{\e_k\}_{k=1}^{\infty}$} in {$\H$}. 
If {$A \in \Tr(\H)$}, then the {\it trace} of {$A$} is  {$\trace(A) = \sum_{k=1}^{\infty}\la \e_k, A\e_k\ra$},
which is independent of the orthonormal basis.

{\bf Positive operator-valued measures}. 
Let {$(\X, \Sigma)$} be a measurable space, where {$\X$} is a non-empty set and {$\Sigma$} is a $\sigma$-algebra of subsets of {$\X$}. 
A {$\Sym^{+}(\H)$}-valued measure {$\mu$} is a {\it countably additive}\footnote{Falb \citep{Falb:1969} used {\it weakly countably additive} vector measures, which are in fact {\it countably additive} \citep{Diestel:Sequences}.} function {$\mu: \Sigma \mapto \Sym^{+}(\H)$}, with {$\mu(\emptyset) = 0$}, so that for any 
sequence of pairwise disjoint subsets {$\{A_j\}_{j=1}^{\infty}$} in {$\Sigma$},
\begin{align}
\mu(\cup_{j=1}^{\infty}A_j) = \sum_{j=1}^{\infty}\mu(A_j),\;\;\; \text{which converges in the operator norm on $\L(\H)$}.
\end{align}
To state Bochner's Theorem for operator-valued measures, we need the notions of finite {$\Sym^{+}(\H)$}-valued measure and ultraweak continuity. 
Let {$\X = \R^n$} (a locally compact space in general). A {\it finite} {$\Sym^{+}(\H)$}-valued Radon measure is a {$\Sym^{+}(\H)$}-valued measure such that for any operator
{$A \in \Sym^{+}(\H)\cap \Tr(\H)$}, the scalar measure
\begin{align}
\mu_A: \Sigma \mapto \R^{+}, \;\;\; \mu_{A}(B) = \trace(A\mu(B)), \;\;\; B \in \Sigma,
\end{align}
is a finite positive Radon measure on {$\R^n$}.
A function {$k:\R^n \mapto \L(\H)$} is said to be {\it ultraweakly continuous} if for each operator {$A \in \Tr(\H)$},
the following scalar function is continuous
\begin{align}
k_A: \R^n \mapto \R, \;\;\; k_A(x) = \trace(Ak(x)).
\end{align}
The following is then the generalization of Bochner's Theorem to the vector-valued setting.
\begin{theorem}
[\textbf{Operator-valued Bochner Theorem} \citep{Neeb:Operator1998}]
\label{theorem:Bochner-operator}
An ultraweakly continuous 
function {$k: \R^n \mapto \L(\H)$} is positive definite if and only if there exists a finite {$\Sym^{+}(\H)$}-valued measure {$\mu$} on 
{$\R^n$} 
such that
\begin{align}
\label{equation:k-expression1}
k(x) = \hat{\mu}(x) = \int_{\R^n}\exp(i \la \omega,x\ra)d\mu(\omega) = \int_{\R^n} \exp(-i \la \omega,x\ra)d\mu(\omega).
\end{align}
The Radon measure {$\mu$} is uniquely determined by {$\hat{\mu} = K$}.
\end{theorem}
{\bf General case}. The above version of Bochner's Theorem holds 
in a much more general setting, where $\R^n$ is replaced by a locally compact abelian group $G$. For  the general version, we refer to \citep{Neeb:Operator1998}.

{\bf Determining {$\mu$} 
from {$k$}}. In order to compute 
feature maps using Bochner's Theorem,   
we need to compute $\mu$ 
from the given operator-valued function $k$. 
Suppose that the density function $\mu(\omega)$ of $\mu$ with respect to the Lebesgue measure on $\R^n$ exists. 
Let {$\{\e_j\}_{j=1}^{\infty}$} be any orthonormal basis for {$\H$}. For any vector {$\a = \sum_{j=1}^{\infty}a_j\e_j\in \H$}, we have
\begin{align*}
\mu(\omega)\a = \sum_{j=1}^{\infty}\la \e_j, \mu(\omega)\a\ra\e_j = \sum_{j,l=1}^{\infty}a_l\la \e_j, \mu(\omega)\e_l\ra\e_j.
\end{align*}
Thus {$\mu(\omega)$} is completely determined by the infinite matrix of inner products {$(\la \e_j, \mu(\omega)\e_l\ra)_{j,l=1}^{\infty}$}, which can be computed from {$k$} via the inverse Fourier transform {$\F^{-1}$} as follows.
\begin{proposition}
\label{proposition:mu-inversion}
Assume that {$\la \e_j, k(x)\e_l\ra \in L^1(\R^n)$} {$\forall j,l \in \N$}. Then the density function $\mu(\omega)$ of $\mu$ with respect to the Lebesgue measure on $\R^n$ exists and is given by
\begin{align}
\la \e_j, \mu(\omega)\e_l\ra =\Fcal^{-1}[\la \e_j, k(x)\e_l\ra].
\end{align}
\end{proposition}
The positive definite function $k$ gives rise to the shift-invariant positive definite kernel
\begin{align}
\label{equation:Bochner1}
K(x,t) = k(x-t) = \int_{\R^n} \exp(-i \la \omega,x-t\ra)d\mu(\omega).
\end{align}
Similar to the scalar case, using the property {$K(x,t) = K(t,x)^{*}$} and the symmetry of $\mu$, we obtain
\begin{align}
\label{equation:K-expression1}
K(x,t) = \int_{\R^n}\cos(\la \omega, x-t\ra)d\mu(\omega).
\end{align}
In order to generalize the random Fourier feature approach to the operator-valued kernel {$K(x,t)$}
we need to construct a probability measure $\rho$ on {$\R^n$} such that {$K(x,t)$}
is the expectation of an operator-valued random variable with respect to $\rho$. Equivalently, we need to factorize the density $\mu(\omega)$ as
\begin{align}
\label{equation:factorize}
\mu(\omega) = \tilde{\mu}(\omega)\rho(\omega),
\end{align}
where $\tilde{\mu}(\omega)$ is a finite {$\Sym^{+}(\H)$}-valued function on {$\R^n$} and $\rho(\omega)$ is the density function of 
the probability measure
$\rho$.

\begin{remark} Throughout the rest of the paper, we assume that $k$ satisfies the assumptions of Proposition \ref{proposition:mu-inversion}.
We then identify the measures $\mu$ and $\rho$ by their density functions  $\mu(\omega)$ and $\rho(\omega)$, respectively,
with respect to the Lebesgue measure.
\end{remark}

{\bf Key differences between the scalar and operator-valued settings}. 
Before proceeding with the probability measure and feature map construction, we point out {\it two key differences} between the scalar and operator-valued settings.
\begin{enumerate}  
\item In the scalar setting, with normalization, $\mu$ is a probability measure {\it uniquely determined} via $\hat{\mu} = k$. In the operator-valued setting, the operator-valued measure $\mu$ is also uniquely determined by $k$, as stated in Proposition \ref{proposition:mu-inversion}. However, the probability measure $\rho$ in Eq.~(\ref{equation:factorize}) needs to be {\it explicitly} constructed, that is it is {\it not} automatically determined by $k$. We present a general formula for computing $\rho$ in Section \ref{section:probability-construction}.

\item The factorization stated in Eq.~(\ref{equation:factorize}) is generally {\it non-unique}. As we show below, in general, there are many (in fact, potentially infinitely many)
pairs {$(\tilde{\mu}, \rho)$} such that Eq.~(\ref{equation:factorize}) holds. Thus there are generally (infinitely) many operator-valued feature maps 
corresponding to the operator-valued version of Bochner's Theorem. We illustrate this property via examples in Sections \ref{section:feature-construction} and \ref{section:probability-construction} below.
\end{enumerate}


\subsection{Formal Construction of Approximate Fourier Feature Maps}
\label{section:feature-construction}
Assuming for the moment that we have a pair $(\tilde{\mu}, \rho)$ satisfying the factorization in Eq.~(\ref{equation:factorize}), 
then Eq.~(\ref{equation:K-expression1}) takes the form
{
\begin{align}
K(x,t) &= \int_{\R^n}\cos(\la \omega, x-t\ra) \tilde{\mu}(\omega)d\rho(\omega)
= \bE_{\rho}[\cos(\la \omega, x-t\ra)\tilde{\mu}(\omega)].
\end{align}
}
Let {$\{\omega_j\}_{j=1}^D$}, {$D \in \N$}, be $D$ points in $\R^n$ randomly sampled independently from {$\rho$}. 
Then {$K(x,t)$} can be approximated by by the empirical sum
{
\begin{align}
\label{equation:KD}
\hat{K}_D(x,t) &= \hat{k}_D(x-t) = \frac{1}{D}\sum_{l=1}^D\cos(\la \omega_l, x-t\ra)\tilde{\mu}(\omega_l)
\nonumber
\\
& = \frac{1}{D}\sum_{l=1}^D[\cos(\la \omega_l, x\ra)\cos(\la \omega_l, t)]\tilde{\mu}(\omega_l)
+ \frac{1}{D}\sum_{l=1}^D\sin(\la \omega_l, x\ra)\sin(\la \omega_l, t\ra)]\tilde{\mu}(\omega_l).
\end{align}
}
Let {$\F$} be a separable Hilbert space and {$\psi:\R^n \mapto \L(\H,\F)$} be such that
\begin{align}
\label{equation:mu-decomp}
\tilde{\mu}(\omega) = \psi(\omega)^{*}\psi(\omega), \;\;\;\; \psi(\omega): \H \mapto \F
\end{align}
Such a pair {$(\psi, \F)$} always exists, with one example being {$\F = \H$} and {$\psi(\omega) = \sqrt{\tilde{\mu}(\omega)}$}.

\begin{remark}
As we demonstrate via the examples below, the decomposition $\tilde{\mu}(\omega) = \psi(\omega)^{*}\psi(\omega)$ is also generally non-unique, which 
is another reason for the non-uniqueness of the approximate feature maps.
\end{remark}

{\bf Operator-valued Fourier feature map}. The decompositions for {$\hat{K}_D$} in Eqs.~(\ref{equation:KD}) and (\ref{equation:mu-decomp}) immediately give us the following approximate feature map
\begin{align}
\label{equation:feature-general}
\hat{\Phi}_D(x) = \frac{1}{\sqrt{D}}
\begin{pmatrix}
\cos(\la \omega_1, x\ra)\psi(\omega_1)\\
\sin(\la \omega_1, x\ra)\psi(\omega_1)\\
\cdots\\
\cos(\la \omega_D, x\ra)\psi(\omega_D)\\
\sin(\la \omega_D, x\ra)\psi(\omega_D)
\end{pmatrix}
:\H \mapto \F^{2D}.
\end{align}
with
\begin{align}
K_D(x,t) = [\hat{\Phi}_D(x)]^{*}[\hat{\Phi}_D(t)].
\end{align}

%
{\bf Special cases}. For $\H=\R$, we have $\tilde{\mu} = 1$ (assuming normalization) and $\rho = \mu$, and we thus recover the Fourier features in the scalar setting.
For $\H = \R^d$, for some $d \in \N$, we obtain the feature map in \citep{Romain:2016}.

\subsection{Probability Measure and Feature Map Construction in Some Special Cases}
\label{section:special-construction}

We first consider several examples of operator-valued kernels arising from scalar-valued kernels. For these examples, both the $\Sym^{+}(\H)$-valued measure $\mu$ and the probability measure $\rho$ can be derived from the corresponding probability measure for the scalar kernels.
These examples have also been considered by \citep{Romain:2016}, however we treat them in greater depth here, particularly the curl-free and div-free kernels (see detail below). One important aspect that we note is that the approach for computing the probability measure $\rho$ in this section is specific for each kernel and does not generalize to a general kernel.   
We return to these examples in the general setting of Section \ref{section:probability-construction}, where we present a general formula for computing $\rho$ for a general kernel $k$.

\begin{example}[\textbf{Separable kernels}]
\end{example}
Consider the simplest case, where the operator-valued positive definite function $k$ has the form
\begin{align}
k(x) = g(x)A,
\end{align}
where {$A \in \Sym^{+}(\H)$} and {$g:\R^n \mapto \R$} is a scalar-valued positive definite function. Let $\rho_0$ be the probability measure on $\R^n$ such that
{$g(x) = \bE_{\rho_0}[e^{- i \la \omega, x\ra}]$}. It follows immediately that
\begin{align}
k(x) = \int_{\R^n}e^{-i \la \omega, x\ra}d\mu(\omega) \;\;\; \text{where}\;\;\; \mu(\omega) = A \rho_0(\omega).
\end{align}
Thus we can set 
\begin{align}
\tilde{\mu}(\omega) = A, \;\;\; \rho = \rho_0.
\end{align}
For the operator {$\psi(\omega)$} in Eq.~(\ref{equation:mu-decomp}), we can set either
\begin{align}
\psi(\omega) = \sqrt{A},
\end{align} 
or, if {$A$} is a symmetric positive definite matrix, we can also compute {$\psi$} via the Cholesky decomposition of {$A$} by setting
\begin{align}
\psi(\omega) = U, \;\;\; \text{where}\;\;\; A = U^TU,
\end{align}
with $U$ being an upper triangular matrix. Thus in this case, with the probability measure $\rho = \rho_0$, there are {\it at least two choices} for the feature map
$\hat{\Phi}_D$, each resulting from one choice of $\psi(\omega)$ as discussed above. In practice, a particular $\psi$ should be chosen based on its computational complexity, which in turn depends on the structure of $A$ itself.

\begin{example}[\textbf{Curl-free and divergence-free kernels}]
\end{example}
Consider next the matrix-valued curl-free and divergence kernels in \citep{Fuselier2006}. 
In \citep{Romain:2016}, the authors present what we call the {\it unbounded feature maps} below for these kernels, without, however, the analytical expression for 
the feature map of the div-free kernel. We now present the analytical expressions for the feature maps for both these kernels. More importantly, we show
that, apart from the unbounded feature maps, there are generally {\it infinitely many bounded feature maps}  associated with these kernels. 

Let {$\phi$} be a scalar-valued twice-differentiable positive definite function on {$\R^n$}.
Let {$\nabla$} denote the {$n \times 1$} gradient operator and 
{$\Delta = \nabla^T \nabla$} denote the Laplacian operator. Define
{
\begin{equation}
k_{\div} = (-\Delta I_n + \nabla \nabla^T)\phi,\;\;\;k_{\curl} = - \nabla \nabla^T \phi.
\end{equation}
}
Then {$k_{\div}$} and {$k_{\curl}$} are {$n \times n$} matrices, whose columns are divergence-free and curl-free functions, respectively. 
The functions
{$k_{\curl}$} and {$k_{\div}$} give rise to the corresponding  positive definite kernels
\begin{align*}
K_{\curl}(x,t) = k_{\curl}(x-t),\;\;\; \text{and}\;\;\; K_{\div}(x,t) = k_{\div}(x-t).
\end{align*}
For the Gaussian case {$\phi(x) = \exp(-\frac{||x||^2}{\sigma^2})$}, the functions $k_{\curl}$ and $k_{\div}$ are given by
\begin{align}
\label{equation:curl-div-k}
k_{\curl}(x) &= \frac{2}{\sigma^2}\exp(-\frac{||x||^2}{\sigma^2})[I_n - \frac{2}{\sigma^2}xx^T].
\\
k_{\div}(x) &= \frac{2}{\sigma^2}\exp(-\frac{||x||^2}{\sigma^2})
[((n-1) - \frac{2}{\sigma^2}||x||^2)I_n + \frac{2}{\sigma^2}xx^T].
\end{align}

\begin{lemma}
\label{lemma:curl-div}
Let $\rho_0$ be the probability measure on {$\R^n$} such that {$\phi(x) = \bE_{\rho_0}[e^{-i \la \omega, x\ra}] = \hat{\rho_0}(x)$}. Then, under the condition
{$\int_{\R^n}||\omega||^2d\rho_0(\omega)  < \infty$},
we have
\begin{align}
k_{\curl}(x) &= \int_{\R^n}e^{-i \la \omega, x\ra}\omega\omega^T\rho_0(\omega)d\omega = \int_{\R^n}e^{-i \la \omega, x\ra} (\mu_{\curl})(\omega)d\omega.
\\
k_{\div}(x) &= \int_{\R^n}e^{-i \la \omega, x\ra}[||\omega||^2I_n - \omega\omega^T]\rho_0(\omega)d\omega = \int_{\R^n}e^{-i \la \omega, x\ra} (\mu_{\div})(\omega)d\omega,
\end{align}
where {$\mu_{\curl}(\omega) = \omega\omega^T\rho_0(\omega)$} and {$\mu_{\div}(\omega) = [||\omega||^2I_n - \omega\omega^T]\rho_0(\omega)$}.
\end{lemma}
The condition {$\int_{\R^n}||\omega||^2d\rho_0(\omega) < \infty$} in Lemma \ref{lemma:curl-div} guarantees that $\phi$ is twice-differentiable, which
is the underlying assumption for curl-free and divergence-free kernels. 

{\bf Unbounded feature maps}. Consider first  the curl-free kernel. From the expression {$\mu_{\curl}(\omega) = \omega\omega^T\rho_0(\omega)$}, we immediately see that for the factorization in Eq.~(\ref{equation:factorize}), we can set  
{
\begin{align*}
\mu_{\curl}(\omega) = \tilde{\mu}(\omega) \rho(\omega), \;\;\; \text{with} \;\;\; \tilde{\mu}(\omega) = \omega\omega^T, \rho = \rho_0. 
\end{align*}
}
For the Gaussian case, {$\rho_0(\omega) = \frac{(\sigma \sqrt{\pi})^n}{(2\pi)^n} e^{-\frac{\sigma^2||\omega||^2}{4}} \sim\Ncal(0, \frac{2}{\sigma^2}I_n)$}. 
In Eq.~(\ref{equation:mu-decomp}), we can set 
{
\begin{align}
\label{equation:curl-unbounded}
\psi(\omega) = \omega^T, \;\;\;\text{so that}\;\;\; \hat{\Phi}_D(x) \text{\;\;is a matrix of size\;\;} 2D \times n.
\end{align}
}
We can also set 
\begin{align}
\label{equation:curl-unbounded-2}
\psi(\omega) = \sqrt{\omega \omega^T} = \frac{\omega \omega^T}{||\omega||},\;\;\;\text{so that} \;\;\; \hat{\Phi}_D(x) \text{\;\;is a matrix of size\;\;} 2Dn \times n. 
\end{align}
Clearly the choice for {$\psi(\omega)$} in Eq.~(\ref{equation:curl-unbounded}) is preferable computationally to that in Eq.~(\ref{equation:curl-unbounded-2}). One thing that can be observed immediately is that both {$\mu_{\curl}$} and {$\psi$} are {\it unbounded} functions of {$\omega$}, which complicates the convergence analysis of the corresponding kernel approximation (see Section \ref{section:convergence} for further discussion).

{\bf Bounded feature maps}. The unbounded feature maps above correspond to one particular choice of the probability measure $\rho$, namely $\rho = \rho_0$. 
However, this is not the only valid choice for $\rho$. We now exhibit another choice for $\rho$ that results in a bounded feature map, whose convergence behavior is much simpler to analyze.
Consider the Gaussian case, with {$\rho_0$} as given above. Clearly, we can choose for another factorization of {$\mu_{\curl}$} the factors
{
\begin{align}
\tilde{\mu}(\omega) = \omega\omega^T e^{-\frac{\sigma^2||\omega||^2}{8}}2^{n/2},\;\;\; \rho(\omega) = \frac{1}{2^{n/2}}\frac{(\sigma \sqrt{\pi})^n}{(2\pi)^n}e^{-\frac{\sigma^2||\omega||^2}{8}} \sim \Ncal(0, \frac{4}{\sigma^2}I_n).
\end{align}
}
Then {$\tilde{\mu}(\omega)$} is a bounded function of {$\omega$}, with the corresponding bounded map
\begin{align}
\label{equation:curl-bounded}
\psi(\omega) = \omega^T e^{-\frac{\sigma^2||\omega||^2}{16}}2^{n/4}.
\end{align}
For the divergence-free kernel, we have
{$\sqrt{||\omega||^2I_n - \omega\omega^T} = ||\omega||I_n - \frac{\omega\omega^T}{||\omega||}$}, giving
the corresponding maps
\begin{align}
\label{equation:div-feature}
\psi(\omega) &=(||\omega||I_n - \frac{\omega\omega^T}{||\omega||}) & \;\;\text{(unbounded feature map)},\;\;
\\
\psi(\omega) &=(||\omega||I_n - \frac{\omega\omega^T}{||\omega||})e^{-\frac{\sigma^2||\omega||^2}{16}}2^{n/4} & \;\; \text{(bounded feature map)}.
\end{align}
Since there are infinitely many ways to split the Gaussian function {$e^{-\frac{\sigma^2||\omega||^2}{4}}$} into a product of two Gaussian functions, 
it follows that there are {\it infinitely many bounded Fourier feature maps} associated with both the curl-free and div-free kernels induced by the Gaussian kernel.
We show below that, under appropriate conditions on $k$, bounded feature maps always exist.

\subsection{First Main Result: Probability Measure Construction in the General Case} 
\label{section:probability-construction}
For the separable and curl-free and div-free kernels, we obtain a probability measure {$\rho$}  directly from the corresponding scalar-valued kernels. We now show
how to construct {$\rho$} given a general {$k$}, under appropriate assumptions on {$k$}. Furthermore, we show
that the corresponding feature map is {\it bounded}, in the sense that {$\tilde{\mu}(\omega)$} is a bounded function of {$\omega$} (see the precise statement in Corollary \ref{corollary:measure-trace}).
 
\begin{proposition}
\label{proposition:Bochner-projection}
Let {$K:\R^n \times \R^n \mapto \L(\H)$} be an ultraweakly continuous shift-invariant positive definite kernel. Let $\mu$ be the unique finite
 {$\Sym^{+}(\H)$}-valued measure satisfying Eq.~(\ref{equation:Bochner1}). Then {$\forall \a\in \H$}, {$\a \neq 0$}, the scalar-valued kernel defined by {$K_{\a}(x,t) = \la \a, K(x,t)\a\ra$} is positive definite. Furthermore, there exists a unique finite positive Borel measure {$\mu_{\a}$} on {$\R^n$} such that $K_{\a}$ is the Fourier transform of $\mu_{\a}$, that is
\begin{align}
K_{\a}(x,t) = \int_{\R^n}\exp(-i \la \omega,x-t\ra)d\mu_{\a}(\omega).
\end{align}
The measure $\mu_{\a}$ is given by 
\begin{align}
\mu_{\a}(\omega) = \la \a, \mu(\omega)\a\ra,\;\;\; \omega \in \R^n.
\end{align}
\end{proposition}
Let {$\{\e_j\}_{j=1}^{\infty}$} be any orthonormal basis for {$\H$}. 
By Proposition \ref{proposition:Bochner-projection}, $\forall j \in \N$, the scalar-valued kernel
\begin{align}
K_{jj}(x,t) = K_{\e_j}(x,t) = \la \e_j, K(x,t)\e_j\ra
\end{align}
is positive definite and is the Fourier transform of the finite positive Borel measure
\begin{align}
\mu_{jj}(\omega) = \la \e_j, \mu(\omega)\e_j\ra, \;\;\; \omega \in \R^n.
\end{align}
The measures {$\mu_{jj}$, $j \in \N$}, which depend on the choice of orthonormal basis {$\{\e_j\}_{j\in \N}$}, collectively
give rise to the following measure, which is independent of {$\{\e_j\}_{j\in \N}$}.

\begin{theorem}
[\textbf{Probability Measure Construction}]
\label{theorem:measure-trace}
Assume that the positive definite function $k$ in Bochner's Theorem satisfies: 
(i) {$k(x) \in \Tr(\H)$ $\forall x \in \R^n$}, and 
(ii) {$\int_{\R^n}|\trace[k(x)]|dx < \infty$}. Then its corresponding finite {$\Sym^{+}(\H)$}-valued measure $\mu$ satisfies
\begin{align}
\mu(\omega) \in \Tr(\H) \;\;\;\forall \omega \in \R^n,\;\;\;
\trace[\mu(\omega)] \leq \frac{1}{(2\pi)^n}\int_{\R^n}|\trace[k(x)]|dx.
\end{align}
The following is a finite positive Borel measure on $\R^n$
\begin{align}
\mu_{\trace}(\omega) &= \trace(\mu(\omega)) = \sum_{j=1}^{\infty}\mu_{jj}(\omega)
= \frac{1}{(2\pi)^n}\int_{\R^n}\exp(i \la \omega, x\ra)\trace[k(x)]dx.
\end{align}
The normalized measure
{
$\frac{\mu_{\trace}(\omega)}{\trace[k(0)]}$
}
is a probability measure on $\R^n$. 
\end{theorem}

{\bf Special case}. For $\H = \R$, we obtain 
\begin{align}
\mu_{\trace}(\omega) 
= \frac{1}{(2\pi)^n}\int_{\R^n}\exp(i \la \omega, x\ra)k(x)dx = \mu(\omega),
\end{align}
so that the scalar-setting is a special case of Theorem \ref{theorem:measure-trace}, as expected.

\begin{corollary}
\label{corollary:measure-trace}
Under the hypothesis of Theorem \ref{theorem:measure-trace}, in Eq.~(\ref{equation:factorize}) we can set
\begin{align}
\label{equation:constructionII}
\rho(\omega) = \frac{\mu_{\trace}(\omega)}{\trace[k(0)]}, \;\;\; 
\tilde{\mu}(\omega) = 
\left\{
\begin{matrix}
\trace[k(0)]\frac{\mu(\omega)}{\mu_{\trace}(\omega)}, & \mu_{\trace}(\omega) > 0\\
0, & \mu_{\trace}(\omega) = 0.
\end{matrix}
\right.
\end{align}
The function $\tilde{\mu}$ in Eq.~(\ref{equation:constructionII}) satisfies {$\tilde{\mu}(\omega) \in \Sym^{+}(\H)$} and has bounded trace, i.e.
{
\begin{align}
||\tilde{\mu}(\omega)||_{\trace} = \trace[\tilde{\mu}(\omega)] \leq \trace[k(0)] \;\;\;\forall \omega \in \R^n.
\end{align}
}
\end{corollary}

Let us now illustrate Theorem \ref{theorem:measure-trace} and \ref{corollary:measure-trace} on the separable kernels and curl-free and div-free kernels. We  note that for the separable kernels, we obtain the same probability measure $\rho$ as in Section \ref{section:special-construction}. However, for the curl-free and div-free kernels, we obtain a different probability measure compared to Section \ref{section:special-construction}, which illustrates the non-uniqueness of $\rho$.
 
\begin{example}[\textbf{Separable kernels}]
\end{example}
For the {\it separable kernels} of the form $k(x) = g(x)A$, with {$A \in \Sym^{+}(\H) \cap \Tr(\H)$} and $g \in L^1(\R^n)$, we have
\begin{align}
\mu_{\trace}(\omega) = \trace(A)\frac{1}{(2\pi)^n}\int_{\R^n}\exp(i \la \omega, x\ra)g(x)dx = \trace(A)\rho_0(\omega).
\end{align}
Since {$\trace[k(0)] = \trace(A)$}, we recover $\rho(\omega) = \rho_0(\omega)$. Here {$\tilde{\mu}(\omega) = A$} and {$||\tilde{\mu}(\omega)||_{\trace} = \trace(A) < \infty$}.

\begin{example}[\textbf{Curl-free and div-free kernels}]
\end{example}
For the {\it curl-free kernel}, we have {$\mu(\omega) = \omega\omega^T\rho_0(\omega)$} and thus 
\begin{align}
\mu_{\trace}(\omega) = ||\omega||^2\rho_0(\omega),
\end{align}
which is
a finite measure by the assumption {$\int_{\R^n}||\omega||^2d\rho_0(\omega) < \infty$}. For the Gaussian case, since {$\trace[k(0)] = \frac{2n}{\sigma^2}$}, the corresponding probability measure is 
\begin{align}
\rho(\omega) = \frac{\sigma^2}{2n}||\omega||^2\rho_0(\omega).
\end{align}
Similarly, for the {\it div-free kernel}, we have {$\mu(\omega) = [||\omega||^2I_n - \omega\omega^T]\rho_0(\omega)$} and thus 
\begin{align}
\mu_{\trace}(\omega) = (n-1)||\omega||^2\rho_0(\omega). 
\end{align}
For the Gaussian case, since {$\trace[k(0)] = \frac{2n(n-1)}{\sigma^2}$}, the corresponding probability measure is 
\begin{align}
\rho(\omega) = \frac{\sigma^2}{2n}||\omega||^2\rho_0(\omega).
\end{align}
Clearly, for both the curl-free and div-free kernels, the probability measure {$\rho$} is {\it non-unique}. We can, for example, obtain {$\rho$} by normalizing the measure
\begin{align}
(1+||\omega||^2)\rho_0(\omega)
\end{align}
and the corresponding {$\tilde{\mu}(\omega)$} still has bounded trace. 

\begin{example}[\textbf{Sum of kernels}]
\end{example}
Consider now the probability measure and feature maps corresponding to the sum of two kernels, which is readily generalizable to any finite sum of kernels. Let $k_1$, $k_2$ be two positive definite functions satisfying the assumptions of Theorem \ref{theorem:measure-trace}, which are the Fourier transforms of two $\Sym^{+}(\H)$-valued measures $\mu_1$ and $\mu_2$, respectively. Then their sum $k = k_1 + k_2$ is clearly the Fourier transform of $\mu = \mu_1 + \mu_2$. Then the probability measure $\rho$ and the function $\tilde{\mu}$ corresponding to $k$ is given by  
\begin{align}
\rho(\omega) &= \frac{\mu_{\trace}(\omega)}{\trace[k(0)]} = \frac{\mu_{1,\trace}(\omega) + \mu_{2,\trace}(\omega)}{\trace[k_1(0)] + \trace[k_2(0)]}, 
\\
\tilde{\mu}(\omega) &= 
\left\{
\begin{matrix}
(\trace[k_1(0)] + \trace[k_2(0)])\frac{\mu_1(\omega) + \mu_2(\omega)}{\mu_{1,\trace}(\omega) + \mu_{2,\trace}(\omega)}, & \mu_{1,\trace}(\omega) + \mu_{2,\trace}(\omega) > 0,\\
0, & \mu_{1,\trace}(\omega) + \mu_{2, \trace}(\omega) = 0.
\end{matrix}
\right.
\end{align}
We the obtain the Fourier feature map for the kernel $K(x,t) = k(x-t)$ using Eqs.~(\ref{equation:mu-decomp}) and (\ref{equation:feature-general}).

We contrast this approach with the following concatenation approach. Let $(\Phi_{K_j}, \F_{K_j})$ be the feature maps associated with $K_j(x,t) = k_j(x-t)$, $j =1,2$.
Let $\F_K$ be the direct Hilbert sum of $\F_{K_1}$ and $\F_{K_2}$. Consider the map $\Phi_K:\X \mapto \mathcal{L}(\H,\F_K)$, $\F_K = \F_{K_1} \oplus \F_{K_2}$, defined by 
\begin{align}
\Phi_K(x) = 
\left(
\begin{matrix}
\Phi_{K_1}(x)
\\
\Phi_{K_2}(x)
\end{matrix}
\right),\;\;\;
\Phi_K(x)w = 
\left(
\begin{matrix}
\Phi_{K_1}(x)w
\\
\Phi_{K_2}(x)w
\end{matrix}
\right),\;\;\; w\in \H,
\end{align}
which essentially stacks to the two maps $\Phi_{K_1}$, $\Phi_{K_2}$ on top of each other. Then clearly 
\begin{align*}
\Phi_K(x)^{*}\Phi_K(t) = \Phi_{K_1}(x)^{*}\Phi_{K_1}(t) + \Phi_{K_2}(x)^{*}\Phi_{K_2}(t) = K_1(x,t) + K_2(x,t) = K(x,t),
\end{align*}
so that $(\Phi_K, \F_K)$ is a feature map representation for $K = K_1 + K_2$. 
If $\dim(\F_{K_j}) < \infty$, then we have $\dim({\F_K})  = \dim(\F_{K_1}) + \dim(F_{K_2})$. 
Thus, from a practical viewpoint, this approach can be computationally expensive, since the dimension of the feature map for the sum kernel can be very large, especially if we have a sum of many kernels.

\subsection{Second Main Result: Uniform Convergence Analysis}
\label{section:convergence}

Having computed the approximate version {$\hat{K}_D$} for {$K$}, we need to show that this approximation is consistent, that is
{$\hat{K}_D$} approaches {$K$} in some sense, as {$D \approach \infty$}. Since {$K(x,t) = k(x-t)$} it suffices for us to consider the convergence of {$\hat{k}_D$} towards {$k$}. 

Recall the Hilbert space of Hilbert-Schmidt operators $\HS(\H)$, that is of bounded operators on $\H$ satisfying
\begin{align*}
||A||^2_{\HS} = \trace(A^{*}A) = \sum_{j=1}^{\infty}||A\e_j||^2 < \infty,
\end{align*}
for any orthonormal basis $\{\e_j\}_{j=1}^{\infty}$ in $\H$. Here
$||\;||_{\HS}$ denotes the Hilbert-Schmidt norm, which is induced by the Hilbert-Schmidt inner product
\begin{align*}
\la A, B\ra_{\HS} = \trace(A^{*}B) = \sum_{j=1}^{\infty}\la A\e_j, B\e_j\ra,\;\;\; A,B \in \HS(\H).
\end{align*}
In the following, we assume that {$k(x)\in \HS(\H)$}. Since we have shown that, under appropriate assumptions, bounded feature maps always exist, we focus exclusively on analyzing the convergence associated with them. Specifically, we  show that for bounded feature maps, for any compact set {$\Omega \subset \R^n$}, we have
\begin{align}
\sup_{x \in \Omega}||\hat{k}_D(x) - k(x)||_{\HS} \approach 0 \;\;\;\text{as}\;\;\; D \approach \infty,
\end{align}
with high probability. This generalizes the convergence of {$\sup_{x \in \Omega}|\hat{k}_D(x) - k(x)|$} in the scalar setting. If {$\dim(\H) < \infty$}, then this is convergence in the Frobenius norm {$||\;||_F$}.

\begin{theorem}
[{\bf Pointwise Convergence}]
\label{theorem:convergence-pointwise}
Assume that {$||\tilde{\mu}(\omega)||_{\HS} \leq M$} almost surely and that {$\sigma^2(\tilde{\mu}(\omega)) = \bE_{\rho}[||\tilde{\mu}(\omega)||^2_{\HS}] < \infty$}.
Then for any fixed {$x \in \R^n$},
\begin{align}
\bP[||\hat{k}_D(x) - k(x)||_{\HS} \geq \epsilon] \leq 2 \exp\left(-\frac{D\epsilon}{2M}\log\left[1+\frac{M\epsilon}{\sigma^2(\tilde{\mu}(\omega))}\right]\right) \;\;\;\forall \epsilon > 0.
\end{align}
\end{theorem}

{\bf Assumption 1}. Our uniform convergence analysis requires the following condition
\begin{align}
\mb_1 = \int_{\R^n}||\omega||\;||\mu(\omega)||_{\HS}d(\omega) = \int_{\R^n}||\omega||\;||\tilde{\mu}(\omega)||_{\HS}d\rho(\omega) < \infty.
\end{align}
In the scalar setting, we have {$\tilde{\mu}(\omega) = 1$}, and Assumption 1 becomes 
{
$\int_{\R^n}||\omega||d\rho(\omega) < \infty$,
}
so that {$k$} is differentiable. This is {\it weaker} than the assumptions in \citep{Fourier:NIPS2007,Fourier:UAI2015,Fourier:NIPS2015}, which all require that {
$\int_{\R^n}||\omega||^2d\rho(\omega) < \infty$}, that is $k$ is twice-differentiable.

\begin{theorem}
[{\bf Uniform Convergence}]
\label{theorem:convergence-uniform}
Let {$\Omega \subset \R^n$} be compact with diameter {$\diam(\Omega)$}. Assume that {$||\tilde{\mu}(\omega)||_{\HS} \leq M$} almost surely and that {$\sigma^2(\tilde{\mu}(\omega)) = \bE_{\rho}[||\tilde{\mu}(\omega)||^2_{\HS}] < \infty$}. Then for any $\epsilon > 0$,
\begin{align}
\bP\left(\sup_{x \in \Omega}||\hat{k}_D(x) - k(x)||_{\HS} \geq \epsilon\right)  \leq & a(n)\left(\frac{\mb_1\diam(\Omega)}{\epsilon}\right)^{\frac{n}{n+1}}
\nonumber
\\
& \times \exp\left(-\frac{D\epsilon}{4(n+1)M}\log\left[1+\frac{M\epsilon}{2\sigma^2(\tilde{\mu}(\omega))}\right]\right),
\end{align}
where {$a(n) =  2^{\frac{3n+1}{n+1}}\left(n^{\frac{1}{n+1}} + n^{-\frac{n}{n+1}}\right)$}.
\end{theorem}

\begin{example}[\textbf{Separable kernels}]
\end{example}
For the {\it separable kernel}, Assumption 1 becomes {$||A||_{\HS} < \infty$}, in which case we have uniform convergence with {$M= ||A||_{\HS}$} and {$\sigma^2(\tilde{\mu}) = ||A||_{\HS}^2$}.

\begin{example}[\textbf{Curl-free and div-free kernels}]
\end{example}
For the {\it curl-free kernel}, we have {$||\mu(\omega)||_{\HS} = ||\omega||^2\rho_0(\omega)$}, thus Assumption 1 becomes 
\begin{align}
\int_{\R^n}||\omega||^3d\rho_0(\omega) < \infty,
\end{align}
which, being stronger than the assumption {$\int_{\R^n}||\omega||^2d\rho_0(\omega) < \infty$} in Section ~\ref{section:probability-construction}, guarantees that a bounded feature map can be constructed, with 
\begin{align}
||\tilde{\mu}(\omega)||_{\HS} \leq ||\tilde{\mu}(\omega)||_{\trace}\leq \trace[k(0)]\;\;\ \text{and}\;\; \sigma^2[\tilde{\mu}(\omega)] \leq (\trace[k(0)])^2.
\end{align}
The case of the {\it div-free kernel} is entirely similar.

{\bf Comparison with the convergence analysis in \citep{Romain:2016}}. In \citep{Romain:2016}, the
authors carried out convergence analysis in the spectral norm for matrix-valued kernels. Our results are for 
the more general setting of operator-valued kernels, which induce RKHS of functions with values in a Hilbert space. In this setting,
convergence in the Hilbert-Schmidt norm is {\it strictly stronger} than convergence in the spectral norm. Furthermore, as with previous results in the scalar setting \citep{Fourier:NIPS2007,Fourier:UAI2015,Fourier:NIPS2015}, the convergence in \citep{Romain:2016} also requires twice-differentiable kernels, whereas we require the {\it weaker assumption} of $C^1$-differentiability.

\section{Vector-Valued Learning with Operator-Valued Feature Maps}
\label{section:learning}
Having discussed operator-valued feature maps and their random approximations, we now show how they can be applied in the context of learning in RKHS of vector-valued functions.
Let {$\W,\Y$} be two Hilbert spaces, {$C:\W \mapto \Y$} a bounded operator, {$K:\R^n \times \R^n \mapto \L(\W)$} be a positive definite definite kernel with the corresponding RKHS $\H_K$ of $\mathcal{W}$-valued functions, and {$V$} be a convex loss function. Consider the following general learning problem from \citep{Minh:JMLR2016}
\begin{align}
\label{equation:general}
f_{\z,\gamma} = \argmin_{f \in \H_K} &\frac{1}{l}\sum_{i=1}^lV(y_i, Cf(x_i)) + \gamma_A ||f||^2_{\H_K}
+ \gamma_I \la \f, M\f\ra_{\mathcal{W}^{u+l}}.
\end{align}
Here {$\z = (\x,\y) = \{(x_i,y_i)\}_{i=1}^l\cup\{x_i\}_{i=l+1}^{u+l}$,} {$u, l \in \N$}, with $u,l$ denoting unlabeled and labeled data points, respectively, {$\f = (f(x_j))_{j=1}^{u+l}\in \W^{u+l}$}, {$M:\W^{u+l}\mapto \W^{u+l}$} a positive operator, and {$\gamma_A > 0, \gamma_I > 0$}. 

In \citep{Minh:JMLR2016}, it is shown that the optimization problem (\ref{equation:general}) represents a general learning formulation in RKHS that encompasses supervised and semi-supervised learning via manifold regularization, multi-view learning, and multi-class classification. For the case $V$ is the least square and SVM loss, both binary and multiclass, the solution of  
(\ref{equation:general}) has been obtained in dual form, that is in terms of kernel matrices.

We now present the solution of problem (\ref{equation:general}) in terms of feature map representation, that is in primal form.
Let {$(\Phi_K, \F_K)$} be any feature map for {$K$}.
On the set $\x$, we define the following operator
\begin{align}
\Phi_K(\x) :\W^{u+l} \mapto \F_K, \;\;\;
\Phi_K(\x)\w = \sum_{j=1}^{u+l}\Phi_K(x_j)w_j. 
\end{align}
We also view $\Phi_K(\x)$ as a (potentially infinite) matrix
\begin{align}
\Phi_K(\x) = [\Phi_K(x_1), \ldots, \Phi_K(x_{u+l})]: \W^{u+l} \mapto \F_K,
\end{align}
with the $j$th column being $\Phi_K(x_j)$.
The following is the corresponding version of the Representer Theorem in \citep{Minh:JMLR2016} in feature map representation.

\begin{theorem}
[\textbf{Representer Theorem}]
\label{theorem:representer}
The optimization problem (\ref{equation:general}) has a unique solution
{$f_{\z,\gamma}(x) 
= \sum_{i=1}^{u+l}K(x,x_i)a_i$ for $a_i \in \W$}, {$i=1,\ldots, u+l$}. In terms of feature maps,
{$f_{\z,\gamma}(x) = \Phi_K(x)^{*}\h$}, where
{
\begin{equation}\label{equation:dual-to-primal}
\h = \sum_{i=1}^{u+l}\Phi_K(x_i)a_i = \Phi_K(\x)\a \in \F_K, \;\;\;\; \a = (a_j)_{j=1}^{u+l} \in \W^{u+l}.
\end{equation}
}
\end{theorem}
In the case $V$ is the least square loss, the optimization problem (\ref{equation:general}) has a closed-form solution, which is expressed explicitly in terms of the
operator-valued 
feature map $\Phi_K$. In the following, let {$I_{(u+l) \times l} = [I_l, 0_{l\times u}]^T$} and $J^{u+l}_l= I_{(u+l) \times l}I_{(u+l) \times l}^T$,
which is a $(u+l) \times (u+l)$ diagonal matrix, with
the first $l$ entries on the main diagonal equal to $1$ and the rest being zero. The following is the corresponding version of Theorem 4 in \citep{Minh:JMLR2016} 
in feature map representation.

\begin{theorem}
[\textbf{Vector-Valued Least Square Algorithm}]
\label{theorem:leastsquare-K}
In the case 
{$V$} is the least square loss, that is {$V(y,f(x)) = ||y-Cf(x)||^2_{\Y}$}, 
the solution of the optimization problem (\ref{equation:general}) is 
{$f_{\z,\gamma}(x) = \Phi_K(x)^{*}\h$}, with
{$\h\in \F_K$} 
given by 
\begin{align}
\label{equation:leastsquare-h}
\h =  \left(\Phi_K(\x)[(J^{u+l}_l \otimes C^{*}C) + l\gamma_I M]\Phi_K(\x)^{*} + l \gamma_A I_{\F_K}\right)^{-1}
\Phi_K(\x)(I_{(u+l) \times l} \otimes C^{*})\y.
\end{align}
\end{theorem}

{\bf Comparison with the dual formulation}. In Theorem 4 in \citep{Minh:JMLR2016}, the solution of the 
least square problem above is equivalently given by $f_{\z,\gamma}(x) = \sum_{j=1}^{u+l}K(x,x_j)a_j$, where 
$\a = (a_j)_{j=1}^{u+l} \in \mathcal{W}^{u+l}$ is given by 
\begin{equation}\label{equation:linear-matrix}
(\mathbf{C^{*}C}J^{\mathcal{W}, u+l}_l K[\x] + l \gamma_I MK[\x] + l \gamma_A I_{\mathcal{W}^{u+l}})\a = \mathbf{C^{*}}{\y},
\end{equation}
where $\mathbf{C^{*}} = I_{(u+l) \times l} \otimes C^{*}$ and $K[\x]$ is the $(u+l) \times (u+l)$ operator-valued matrix with the $(i,j)$ 
entry being $K(x_i,x_j)$.

For concreteness, consider the case $\mathcal{W} = \R^d$ for some $d \in \N$. Then Eq.~(\ref{equation:linear-matrix}) is a system of linear equations of size $d(u+l) \times d(u+l)$, which depends only on the dimension $d$ of the output space and the number of data points $(u+l)$. 

If the feature space $\F_K$ is infinite-dimensional, then Eq.~(\ref{equation:leastsquare-h}) is an infinite-dimensional system of linear equations. 

{\bf Approximate feature map vector-valued least square regression}.
Consider now the approximate finite-dimensional feature map $\hat{\Phi}_D(x):\R^d \mapto \R^{2Dr}$, for some $r$, $1 \leq r \leq d$.
Here $r$ depends on the decomposition $\tilde{\mu} = \psi(\omega)^{*}\psi(\omega)$ in Eq.~(\ref{equation:mu-decomp}), with $r =d$ corresponding to e.g. $\psi(\omega) = \sqrt{\tilde{\mu}(\omega)}$. Then instead of the operator 
$\Phi_K(\x):\R^{d(u+l)} \mapto \F_K$, we consider its approximation
\begin{align}
\hat{\Phi}_D(\x) = [\hat{\Phi}_D(x_1), \ldots, \hat{\Phi}_D(x_{u+l})]: \R^{d(u+l)} \mapto \R^{2Dr},
\end{align}
which is a matrix of size {$2Dr \times d(u+l)$}. 
This gives rise to the following system of linear equations, which approximates Eq.~(\ref{equation:leastsquare-h})
\begin{align}
\label{equation:leastsquare-h-approx}
\hat{\h}_D =  \left(\hat{\Phi}_D(\x)[(J^{u+l}_l \otimes C^{*}C) + l\gamma_I M]\hat{\Phi}_D(\x)^{*} + l \gamma_A I_{2Dk}\right)^{-1}
\hat{\Phi}_D(\x)(I_{(u+l) \times l} \otimes C^{*})\y.
\end{align}
Eq.~(\ref{equation:leastsquare-h-approx}) is a system of linear equations of size $2Dr \times 2Dr$, which is independent of the number of data points $(u+l)$. 
This system is more efficient to solve than Eq.~(\ref{equation:linear-matrix}) when 
\begin{align}
2Dr < d(u+l).
\end{align}

\section{Numerical Experiments}
\label{section:experiments}

We report in this section several experiments to illustrate the numerical properties of the feature maps just constructed.
Since the properties of the feature maps for separable kernels follow directly from those of the corresponding scalar kernels, 
we focus here on the curl-free and div-free kernels.

{\bf Approximate kernel computation}. We first checked the quality of the approximation of the kernel values using matrix-valued Fourier feature maps.
Using the standard normal distribution, we generated a set of  {$100$} points in {$\R^3$}, which are normalized to lie in the cube {$[-1,1]^3$}.
On this set, we first computed the curl-free and div-free kernels induced by the Gaussian kernel, based on Eq.~(\ref{equation:curl-div-k}), with {$\sigma=1$}.
We computed the feature maps given by Eq.~(\ref{equation:feature-general}). For the curl-free kernel, in the unbounded map, {$\psi(\omega)$} is given by Eq.~(\ref{equation:curl-unbounded}), with {$\rho(\omega) \sim \Ncal(0, (2/\sigma^2)I_3)$}, and in the bounded map, {$\psi(\omega)$} is given by Eq.~(\ref{equation:curl-bounded}), with 
{$\rho(\omega) \sim \Ncal(0, (4/\sigma^2)I_3)$}. Similarly, for the div-free kernel, the {$\psi(\omega)$} maps, bounded and unbounded, are given by Eq.~(\ref{equation:div-feature}). We then computed
the relative error {$||\hat{K}_D(x,y) - K(x,y)||_F/||K(x,y)||_F$}, with $F$ denoting the Frobenius norm, using {$D=100, 500, 1000$}. The results are reported on Table
\ref{table:kernel-error}.






\begin{table}[t]
  \caption{Kernel approximation by feature maps. The numbers shown are the relative errors measured using the Frobenius norm, averaged over $10$ runs, along with the standard deviations.}
  \label{table:kernel-error}
  \centering
  \begin{tabular}{llll}
    \toprule
    { Kernel}     & {\small$D= 100$}     & {\small$D =500$} & {\small$D = 1000$}\\
    \midrule
    {\small curl-free (bounded)} &  {\small$0.2811$  $(0.0606)$}  &   {\small$0.1011$  $(0.0216)$}&   {\small$0.0906$    $(0.0172)$} \\
    {\small curl-free (unbounded)} &  {\small$0.3315$ $(0.0638)$}  & {\small$0.1363$  $(0.0227)$} &    {\small$0.0984$  $(0.0207)$}    \\
		{\small div-free (bounded)} & {\small$0.2223$ $(0.0605)$} &  {\small$0.1006$   $(0.0221)$}&  {\small$0.0680$  $(0.0114)$}\\
		{\small div-free (unbounded)} & {\small$0.2826$ $(0.0567)$} & {\small$0.1386$ $(0.0388)$} &  {\small$0.0842$   $(0.0167)$}\\
    \bottomrule
  \end{tabular}
\end{table}

\begin{table}[t]
  \caption{Curl-free vector field reconstruction by least square regression using the exact kernel least square regression according to Eq.~(\ref{equation:linear-matrix}) and approximate feature maps according to Eq.~(\ref{equation:leastsquare-h-approx}). The RMSEs, averaged over $10$ runs, are shown with standard deviations.}
  \label{table:field-error}
  \centering
  \begin{tabular}{llll}
    \toprule
    {\small$D$}     & {\small Exact}      & {\small Bounded} & {\small Unbounded} \\
    \midrule
    {\small$D=50$} &  {\small$0.0020$}  & {\small$0.0079$ $(0.0076)$}    &  {\small$0.0254$   $(0.0118)$}\\
    {\small$D=100$} &  {\small$0.0024$}    & {\small$0.0032$  $(0.0024)$}   &   {\small$0.0118$   $(0.0098)$}  \\
    \bottomrule
  \end{tabular}
\end{table}












  





{\bf Vector field reconstruction by approximate feature maps}. Next, we tested the reconstruction of the following curl-free vector field in {$\R^2$},
{$F(x,y) = \sin(4\pi x)\sin^2(2\pi y) \ibf + \sin^2(2\pi x) \sin(4\pi y)\jbf$} on the rectangle {$[-1, -0.4765] \times [-1, -0.4765]$}, 
sampled on a regular grid consisting of {$1600$} points. The reconstruction is done using {$5\%$} of the points on the grid as training data.
We first performed the reconstruction
with exact kernel least square regression according to Eq.~(\ref{equation:linear-matrix}),
using the curl-free kernel induced by the Gaussian kernel, based on Eq.~(\ref{equation:curl-div-k}), with {$\sigma = 0.2$}.
With the same kernel, we then performed the approximate feature map least square regression according to Eq.~(\ref{equation:leastsquare-h-approx}
({$\W=\Y=\R^2, C=I_2, \gamma_I = 0, \gamma_A = 10^{-9}$}), with {$D=50, 100$}.
The results are reported on Table \ref{table:field-error}.










{\bf Discussion of numerical results}.
As we can see from Tables \ref{table:kernel-error} and \ref{table:field-error}, the matrix-valued Fourier feature maps can be used both for approximating
the kernel values as well as directly in a learning algorithm, with increasing accuracy as the feature dimension increases.
Furthermore, while all feature maps associated with a given kernel are essentially equivalent as in Lemma \ref{lemma:f-feature}, we observe that numerically, on average, the bounded feature maps tend to outperform the unbounded maps.

\section{Conclusion and Future Work}
 
We have presented a framework for constructing random operator-valued feature maps
for operator-valued kernels, using the operator-valued version of Bochner's Theorem. We have shown that, due to the non-uniqueness
of the probability measure in this setting, in general many feature maps can be computed, which can be unbounded or bounded. Under certain conditions, which are satisfied for many common kernels such as curl-free and div-free kernels, bounded feature maps can always be computed. We then showed the uniform convergence, with the bounded maps, of the approximate kernel in the Hilbert-Schmidt norm, strengthening previous results in the scalar setting. Finally, we showed how a general vector-valued learning formulation can be expressed in terms of feature maps and demonstrated it experimentally. An extensive empirical evaluation of the proposed formulation is left to future work.

\appendix

\section{Proofs of Main Mathematical Results}
\label{section:proofs}


\begin{proof}{\textbf{of Lemma \ref{lemma:f-feature}}}
For a function $f \in \H_K$ of the form $f = \sum_{j=1}^NK_{x_j}a_j$, $a_j \in \W$, we have
\begin{align*}
f(x) &= \sum_{j=1}^NK(x,x_j)a_j = \sum_{j=1}^N\Phi_K(x)^{*}\Phi_K(x_j)a_j = \Phi_K(x)^{*}\left(\sum_{j=1}^N\Phi_K(x_j)a_j\right),
\\
&=  \Phi_K(x)^{*}\h,
\end{align*}
where
$$
\h = \sum_{j=1}^N\Phi_K(x_j)a_j \in \F_K.
$$
For the norm, we have
\begin{align*}
||f||^2_{\H_K} &= \sum_{i,j=1}^N\la a_i, K(x_i, x_j)a_j\ra_{\W} = \sum_{i,j=1}^N\la a_i, \Phi_K(x_i)^{*}\Phi_K(x_j)a_j\ra_{\W} 
\\
&= \sum_{i,j=1}^N\la \Phi_K(x_i)a_i, \Phi_K(x_j)a_j\ra_{\F_K} = ||\sum_{i=1}^N\Phi_K(x_i)a_i||^2_{\F_K} = ||\h||^2_{\F_K}.
\end{align*}
By letting $N \approach \infty$ in the Hilbert space completion for $\H_K$, it follows that every $f \in \H_K$ has the form
\begin{equation*}
f(x) = \Phi_K(x)^{*}\h,\;\;\; \h \in \F_K,
\end{equation*}
and
$$
||f||_{\H_K} = ||\h||_{\F_K}.
$$
This completes the proof of the lemma.
\end{proof}

\begin{proof}{\textbf{of Proposition \ref{proposition:mu-inversion}}}
For each pair $j,l \in \N$, we have by Bochner's Theorem
\begin{align*}
\la \e_j, k(x)\e_l\ra = \int_{\R^n}\exp(-i\la  \omega, x\ra)\la \e_j, d\mu(\omega)\e_l\ra.
\end{align*}
Thus the proposition follows immediately from the Fourier Inversion Theorem.
\end{proof}

\begin{proof}
{\textbf{of Proposition \ref{proposition:Bochner-projection}}}
Let $\Phi_K: \R^n \mapto \F_K$ be a feature map for $K$, then $K(x,t) = \Phi_K(x)^{*}\Phi_K(t)$.
Then for any set of points $\{x_j\}_{j=1}^N$ and coefficients $\{b_j\}_{j=1}^N$, we have
\begin{align*}
\sum_{j,l=1}^Nb_jb_lK_{\a}(x_j,x_l) 
&= \sum_{j,l=1}^N b_jb_l\la \a, \Phi_K(x_j)^{*}\Phi_K(x_l)\a\ra = \sum_{j,l=1}^Nb_jb_l\la \Phi_K(x_j)\a, \Phi_K(x_l)\a \ra_{\F_K}
\\
&= \sum_{j,l=1}^N\la b_j \Phi_K(x_j)\a, b_l\Phi_K(x_l)\a \ra_{\F_K}
=\left\|\sum_{j=1}^N b_j\Phi_K(x_j)\a\right\|^2_{\F_K} \geq 0.
\end{align*}
This shows that the scalar-valued kernel $K_{\a}$ is positive definite. Thus by the scalar-valued Bochner Theorem, there exists a unique finite positive Borel measure $\mu_{\a}$ on $\R^n$ such that
\begin{align*}
K_{\a}(x,t) = \int_{\R^n}\exp(-i \la \omega,x-t\ra)d\mu_{\a}(\omega).
\end{align*}
From the formulas $K(x,t) = \int_{\R^n}\exp(-i \la \omega,x-t\ra)d\mu(\omega)$ and 
$K_{\a}(x,t) =\la \a, K(x,t)\a\ra$, we obtain $\mu_{\a}(\omega) = \la \a, \mu(\omega) \a\ra$
as we claimed.  
\end{proof}

\begin{proof}{\textbf{of Theorem \ref{theorem:measure-trace}}}
By Proposition \ref{proposition:mu-inversion} and taking into account the fact that $\mu(\omega)$ is
a self-adjoint positive operator on $\H$, we have
\begin{align*}
\trace[\mu(\omega)] &= \sum_{j=1}^{\infty}\la \e_j,\mu(\omega)\e_j\ra = \left|\sum_{j=1}^{\infty}\la \e_j, \mu(\omega)\e_j\ra\right|
\\
&= \frac{1}{(2\pi)^n}\
\left| \sum_{j=1}^{\infty}\int_{\R^n}\exp(i\la \omega, x\ra)\la \e_j, k(x)\e_j\ra dx\right|
\\
&= \frac{1}{(2\pi)^n}\left|\int_{\R^n}\exp(i\la \omega, x\ra)\sum_{j=1}^{\infty}\la \e_j, k(x)\e_j\ra dx \right|
\\
& = \frac{1}{(2\pi)^n}\left|\int_{\R^n}\exp(i \la \omega, x\ra\trace[k(x)]dx \right| \leq \frac{1}{(2\pi)^n}\int_{\R^n}|\trace[k(x)]|dx < \infty.
\end{align*}
This shows that $\mu_{\trace}(\omega) = \trace[\mu(\omega)] = \sum_{j=1}^{\infty}\mu_{jj}(\omega) \in \Tr(\H)$. 
The positivity of $\mu_{\trace}(\omega)$ follows from the positivity of all the $\mu_{jj}$'s. Furthermore, we have
\begin{align*}
\int_{\R^n}d\mu_{\trace}(\omega) = \int_{\R^n}d[\sum_{j=1}^{\infty}\mu_{jj}(\omega)] 
= \sum_{j=1}^{\infty} \la \e_j, k(0)\e_j\ra = \trace[k(0)].
\end{align*}
We note that we must have $\trace[k(0)] > 0$, since $\trace[k(0)] = 0 \equivalent k(0) = 0 \equivalent K(x,x) = 0 \;\forall x \in \R^n \equivalent K(x,t) = 0 \; \forall (x,t) \in \X \times \X$.
It follows that the normalized measure
\begin{align*}
\frac{\mu_{\trace}(\omega)}{\trace[k(0)]}
\end{align*}
is a probability measure on $\R^n$. Since the trace operation is independent of the choice of orthonormal basis
$\{\e_j\}_{j=1}^{\infty}$, it follows that both the normalized and un-normalized measures are independent of the choice
of $\{\e_j\}_{j=1}^{\infty}$. This completes the proof.
\end{proof}

\begin{proof}
{\textbf{of Lemma \ref{lemma:curl-div}}}
We make use of the following property of the Fourier transform (see e.g \citep{Jones:Lebesgue}).
Assume that $f$ and $||x||f$ are both integrable on $\R^n$, then the Fourier transform $\hat{f}$ is differentiable and
\begin{align*}
\frac{\partial \hat{f}}{\partial \omega_j}(\omega) = - \widehat{i x_j f}(\omega), \;\;\; 1 \leq j \leq n.
\end{align*}
Assume further that $||x||^2f$ is integrable on $\R^n$, then this rule can be applied twice to give
\begin{align*}
\frac{\partial^2 \hat{f}}{\partial \omega_j \partial \omega_k}(\omega) = \frac{\partial}{\partial \omega_j}\left(\frac{\partial \hat{f}}{\partial \omega_k}\right) = - \frac{\partial}{\partial \omega_j}[\widehat{i x_k f}] = -\widehat{x_j x_kf}.
\end{align*}

For the curl-free kernel, we have $\phi(x) = \hat{\rho_0}(x)$ and consequently, under the assumption that $\int_{\R^n}||\omega||^2\rho_0(\omega) < \infty$, we have
\begin{align*}
[k_{\curl}(x)]_{jk} = - \frac{\partial^2\phi}{\partial x_j \partial x_k}(x) = -\frac{\partial^2\hat{\rho_0}}{\partial x_j \partial x_k}(x) = \widehat{\omega_j \omega_k \rho_0}(x),\;\;\; 1 \leq j,k \leq n.
\end{align*}
In other words,
\begin{align*}
[k_{\curl}(x)]_{jk} = \int_{\R^n}e^{-i \la \omega, x\ra} \omega_j \omega_k \rho_0(\omega)d\omega.
\end{align*}
It thus follows that
\begin{align*}
k_{\curl}(x) = \int_{\R^n}e^{-i \la \omega, x\ra }\omega \omega^T \rho_0(\omega)d\omega = \int_{\R^n}e^{-i \la \omega, x\ra }\mu(\omega)d\omega,
\end{align*}
where $\mu(\omega) = \omega \omega^T \rho_0(\omega)$. The proof for $k_{\div}$ is entirely similar.
\end{proof}








To prove Theorems \ref{theorem:convergence-pointwise} and \ref{theorem:convergence-uniform} , we need the following concentration result for Hilbert space-valued random variables.

\begin{lemma}
[\citep{smalezhou2007}]
\label{lemma:Pinelis}
Let {$\H$} be a Hilbert space with norm {$||\;||$} and {$\{\xi_j\}_{j=1}^D$}, {$D \in \N$}, be independent random variables with values in {$\H$}. Suppose that for each $j$, {$||\xi_j||\leq M < \infty$} almost surely. Let 
{$\sigma^2_D =\sum_{j=1}^D\bE(||\xi_j||^2)$}. Then
{
\begin{align}
&\bP\left(
\left\|\frac{1}{D}\sum_{j=1}^D[\xi_j - \bE(\xi_j)]\right\| \geq \epsilon
\right)
\leq 2\exp\left(-\frac{D\epsilon}{2M}\log\left[1+\frac{DM\epsilon}{\sigma^2_D}\right]\right) \;\;\; \forall \epsilon > 0.
\end{align}
} 
\end{lemma}

\begin{proof}
{\textbf{of Theorem \ref{theorem:convergence-pointwise}}}
For each $x \in \R^n$ fixed, consider the random variable $\xi(x,,.): (\R^n, \rho) \mapto \Sym(\H)$ defined by
\begin{align*}
\xi(x, \omega) = \cos(\la \omega, x\ra)\tilde{\mu}(\omega).
\end{align*}
We then have
\begin{align*}
\hat{k}_D(x) &= \frac{1}{D}\sum_{j=1}^D\cos(\la \omega_j, x\ra)\tilde{\mu}(\omega_j) = \frac{1}{D}\sum_{j=1}^D\xi(x, \omega_j),
\\
k(x) &= \int_{\R^n}\cos(\la \omega, x\ra)\tilde{\mu}(\omega)d\rho(\omega) = \bE_{\rho}[\xi(x, \omega)].
\end{align*}
Under the assumption that $||\tilde{\mu}(\omega)||_{\HS} \leq M$ almost surely, we also have $||\xi(x,\omega)||_{\HS} \leq M$ almost surely. Its
variance satisfies
\begin{align*}
\sigma^2(\xi(x, \omega)) = \bE_{\rho}||\xi(x, \omega)||^2_{\HS}  \leq E||\tilde{\mu}(\omega)||^2_{\HS} = \sigma^2(\tilde{\mu}(\omega)).
\end{align*}
It follows from Lemma \ref{lemma:Pinelis} that for each fixed $x \in \R^n$, we have
\begin{align*}
\bP[||\hat{k}_D(x) - k(x)||_{\HS} \geq \epsilon]  &= \bP\left(\left\|\frac{1}{D}\sum_{j=1}^D\xi(x,\omega_j) - \bE_{\rho}[\xi(x, \omega)]\right\|_{\HS} \geq \epsilon\right)
\\
&\leq 2 \exp\left(-\frac{D\epsilon}{2M}\log\left[1+\frac{M\epsilon}{\sigma^2(\xi(x, \omega))}\right]\right)
\\
& \leq 2 \exp\left(-\frac{D\epsilon}{2M}\log\left[1+\frac{M\epsilon}{\sigma^2(\tilde{\mu}(\omega))}\right]\right).
\end{align*}
This completes the proof of the theorem.
\end{proof}

To prove Theorem \ref{theorem:convergence-uniform}, we first prove the following preliminary results.



%

\begin{lemma} 
\label{lemma:Lipschitz-1}
Assume that $\int_{\R^n}||\omega||\;||\tilde{\mu}(\omega)||_{\HS}d\rho(\omega) < \infty$. Then the function $k:\R^n \mapto 
\Sym(\H)$, with the latter endowed with the Hilbert-Schmidt norm, is Lipschitz, with
\begin{align}
||k(x) - k(y)||_{\HS} \leq ||x-y||\int_{\R^n}||\omega||\;||\tilde{\mu}(\omega)||_{\HS}d\rho(\omega).
\end{align}
\end{lemma}
\begin{proof}
{\textbf{of Lemma \ref{lemma:Lipschitz-1}}}
Using the fact that the cosine function is Lipschitz with constant $1$, that is $|\cos(x) - \cos(y)| \leq |x-y|$ for all $x, y \in \R$, we have
\begin{align*}
||k(x)-k(y)||_{\HS} &= \left\|\int_{\R^n}[\cos(\la \omega, x\ra)-\cos(\la \omega, y\ra)] \tilde{\mu}(\omega)d \rho(\omega)\right\|_{\HS}
\\
& \leq \int_{\R^n}|\cos(\la \omega, x\ra)-\cos(\la \omega, y\ra)|\;||\tilde{\mu}(\omega)||_{\HS}d\rho(\omega)
\\
& \leq  \int_{\R^n}|\la \omega, x\ra-\la \omega, y\ra|\;||\tilde{\mu}(\omega)||_{\HS}d\rho(\omega)
\\
& \leq ||x-y||\int_{\R^n}||\omega||\;||\tilde{\mu}(\omega)||_{\HS}d\rho(\omega).
\end{align*}
This completes the proof.
\end{proof}

\begin{lemma}
\label{lemma:Lipschitz-2}
The function $\hat{k}_D(x) = \frac{1}{D}\sum_{j=1}^D\cos(\omega_j, x\ra)\tilde{\mu}(\omega_j): \R^n \mapto \Sym(\H)$, with the latter endowed with the Hilbert-Schmidt norm,
is Lipschitz, with
\begin{align}
||\hat{k}_D(x) - \hat{k}_D(y)||_{\HS} \leq  ||x-y||\frac{1}{D}\sum_{j=1}^D||\omega_j||\;||\tilde{\mu}(\omega_j)||_{\HS}.
\end{align}
\end{lemma}
\begin{proof}
{\textbf{of Lemma \ref{lemma:Lipschitz-2}}}
Similar to the proof of Lemma \ref{lemma:Lipschitz-1}, we utilize the fact that $|\cos(x) - \cos(y)| \leq |x-y|$ for all $x, y \in \R$ to arrive at
\begin{align*}
||\hat{k}_D(x) - \hat{k}_D(y)||_{\HS} &= \frac{1}{D}\left\|\sum_{j=1}^D[\cos(\la \omega_j, x\ra) - \cos(\la \omega_j, y\ra)]\tilde{\mu}(\omega_j)\right\|_{\HS}
\\
& \leq ||x-y||\frac{1}{D}\sum_{j=1}^D||\omega_j||\;||\tilde{\mu}(\omega_j)||_{\HS}.
\end{align*}
This completes the proof.
\end{proof}

\begin{corollary} Let $f(x) = \hat{k}_D(x) - k(x): \R^n \mapto \Sym(\H)$, with the latter endowed with the Hilbert-Schmidt norm, then $f$ is Lipschitz, with
\begin{align}
||f(x) - f(y)|| \leq ||x-y||\left(\int_{\R^n}||\omega||\;||\tilde{\mu}(\omega)||_{\HS}d\rho(\omega) + \frac{1}{D}\sum_{j=1}^D||\omega_j||\;||\tilde{\mu}(\omega_j)||_{\HS}\right).
\end{align}
\label{corollary:Lipschitz}
The Lipschitz constant $L_f$ of $f$ satisfies
\begin{align}
\bP(L_f \geq \epsilon) \leq \frac{2\mb_1}{\epsilon},
\end{align}
where $\mb_1 = \int_{\R^n}||\omega||\;||\tilde{\mu}(\omega)||_{\HS}d\rho(\omega)$.
\end{corollary}
\begin{proof}
{\textbf{of Corollary \ref{corollary:Lipschitz}}}
By combing the results of Lemmas \ref{lemma:Lipschitz-1} and \ref{lemma:Lipschitz-2}, we have
\begin{align*}
||f(x) - f(y)||_{\HS} &= ||(\hat{k}_D(x) - k(x)) - (\hat{k}_D(y) - k(y))||_{\HS} \leq ||\hat{k}_D(x) - \hat{k}_D(y)||_{\HS} + ||k(x) - k(y)||_{\HS}
\\
& \leq ||x-y||\left(\int_{\R^n}||\omega||\;||\tilde{\mu}(\omega)||_{\HS}d\rho(\omega) + \frac{1}{D}\sum_{j=1}^D||\omega_j||\;||\tilde{\mu}(\omega_j)||_{\HS}\right)
\end{align*}
as we claimed. Thus $f$ is Lipschitz, with the Lipschitz constant $L_f$ satisfying
\begin{align*}
L_f \leq \int_{\R^n}||\omega||\;||\tilde{\mu}(\omega)||_{\HS}d\rho(\omega) + \frac{1}{D}\sum_{j=1}^D||\omega_j||\;||\tilde{\mu}(\omega_j)||_{\HS},
\end{align*}
with expectation
\begin{align*}
\bE(L_f) \leq 2 \int_{\R^n}||\omega||\;||\tilde{\mu}(\omega)||_{\HS}d\rho(\omega) = 2\mb_1.
\end{align*}
By Markov's inequality, we have for any $\epsilon > 0$,
\begin{align*}
\bP(L_f \geq \epsilon) \leq \frac{\bE(L_f)}{\epsilon} \leq \frac{2\mb_1}{\epsilon}.
\end{align*}
This completes the proof.
\end{proof}

\begin{proof}
{\textbf{of Theorem \ref{theorem:convergence-uniform}}}
For each $r > 0$ fixed, 
let $\Ncal = \Ncal(\Omega, r)$ be the covering number for $\Omega$, that is 
the minimum number of balls $\Omega_j$, $1 \leq j \leq \Ncal$, of radius $r$ covering $\Omega$. 
By Proposition 5 in \citep{CuckerSmale}, the covering number $\Ncal(\Omega,r)$ is bounded above by the expression
\begin{align}
\Ncal = \Ncal(\Omega, r) \leq \left(\frac{2\diam(\Omega)}{r}\right)^n.
\end{align}
Consider the function $f: \R^n \mapto \Sym(\H)$ defined by
\begin{align*}
f(x) = \hat{k}_D(x) - k(x).
\end{align*}
On the ball $\Omega_j$, we have
\begin{align*}
\bP(\sup_{x \in \Omega_j}||\hat{k}_D(x) - k(x)||_{\HS} \geq \epsilon)  = \bP(\sup_{x \in \Omega_j}||f(x)||_{\HS} \geq \epsilon).
\end{align*}
By Corollary \ref{corollary:Lipschitz}, $f$ is a Lipschitz function with Lipschitz constant $L_f > 0$, with $\Sym(\H)$ being endowed with the Hilbert-Schmidt norm. 
Let $x_j$  be the center of the $j$th ball $\Omega_j$. For each $\epsilon > 0$, for any $x \in \Omega_j$, we have
\begin{align*}
||f(x_j) - f(x)||_{\HS} \leq L_f||x_j - x|| \leq rL_f < \frac{\epsilon}{2} \;\;\;\text{when}\;\;\; L_f < \frac{\epsilon}{2r}.
\end{align*}
Since
\begin{align*}
||f(x)||_{\HS} \leq ||f(x) - f(x_j)||_{\HS} + ||f(x_j)||_{\HS},
\end{align*}
we have
\begin{align*}
\sup_{x \in \Omega_j}||f(x)||_{\HS} < \epsilon \;\;\; \text{if}\;\;\; L_f < \frac{\epsilon}{2r}\;\;\;\text{and}\;\;\; ||f(x_j)||_{\HS} < \frac{\epsilon}{2}.
\end{align*}
Thus over the union of balls $\Omega_j$, $1 \leq j \leq \Ncal$, we have
\begin{align*}
\sup_{x \in \cup_{j=1}^{\Ncal}\Omega_j}||f(x)||_{\HS} < \epsilon \;\;\; \text{if}\;\;\; L_f < \frac{\epsilon}{2r}\;\;\;\text{and}\;\;\; ||f(x_j)||_{\HS} < \frac{\epsilon}{2}, \;\;\; 1\leq j \leq \Ncal.
\end{align*}
Thus
\begin{align*}
\bP\left(\sup_{x \in \cup_{j=1}^{\Ncal}\Omega_j}||f(x)||_{\HS} < \epsilon\right)  = \bP\left(L_f < \frac{\epsilon}{2r}\;\;\;\text{and}\;\;\; ||f(x_j)||_{\HS} < \frac{\epsilon}{2}, \;\;\; 1\leq j \leq \Ncal\right).
\end{align*}
We now recall the following properties on an arbitrary probability space $(\Sigma, \bP, \Fcal)$. For any events $A, B$, let $\overline{A}$ denote the complement of $A$ in $\Fcal$ , then we have $\overline{A \cap B} = \overline{A} \cup \overline{B}$, so that 
\begin{align}
\label{equation:probability-complement-1}
\bP(\overline{A \cap B}) &= \bP(\overline{A} \cup \overline{B}) \leq \bP(\overline{A}) + \bP(\overline{B}),
\\
\bP(A \cap B) &= 1- \bP(\overline{A \cap B}) = 1 - \bP(\overline{A} \cup \overline{B}) \geq 1 - \bP(\overline{A}) - \bP(\overline{B}).
\label{equation:probability-complement-2}
\end{align}
Applying property (\ref{equation:probability-complement-2}) with $A = \{L_f < \frac{\epsilon}{2r}\}$ and $B = \{||f(x_j)||_{\HS} < \frac{\epsilon}{2}, 1 \leq j \leq \Ncal\}$, we obtain
\begin{align*}
\bP\left(\sup_{x \in \cup_{j=1}^{\Ncal}\Omega_j}||f(x)||_{\HS} < \epsilon\right) 
&\geq 1 - \bP\left(L_f \geq \frac{\epsilon}{2r}\right) - \bP\left(\overline{\{||f(x_j)||_{\HS} 
< \frac{\epsilon}{2}, 1 \leq j \leq \Ncal\}}\right).
\end{align*}
Applying property (\ref{equation:probability-complement-1}) recursively to the set $\{||f(x_j)||_{\HS} < \frac{\epsilon}{2}, 1 \leq j \leq \Ncal\}$, we obtain
\begin{align*}
\bP\left(\overline{\{||f(x_j)||_{\HS} < \frac{\epsilon}{2}, 1 \leq j \leq \Ncal\}}\right)
\leq \sum_{j=1}^{\Ncal}\bP(\overline{\{||f(x_j)||_{\HS} < \frac{\epsilon}{2}\}}) = \sum_{j=1}^{\Ncal}\bP(\{||f(x_j)||_{\HS} \geq \frac{\epsilon}{2}\}).
\end{align*}
Combining the last two expressions, we have
\begin{align*}
\bP\left(\sup_{x \in \cup_{j=1}^{\Ncal}\Omega_j}||f(x)||_{\HS} < \epsilon\right) 
&\geq 1 - \bP\left(L_f \geq \frac{\epsilon}{2r}\right) - \sum_{j=1}^{\Ncal}\bP(\{||f(x_j)||_{\HS} \geq \frac{\epsilon}{2}\}).
\end{align*}
Equivalently,
\begin{align*}
\bP\left(\sup_{x \in \cup_{j=1}^{\Ncal}\Omega_j}||f(x)||_{\HS} \geq \epsilon\right) 
&\leq \bP\left(L_f \geq \frac{\epsilon}{2r}\right) + \sum_{j=1}^{\Ncal}\bP(\{||f(x_j)||_{\HS} \geq \frac{\epsilon}{2}\}).
\end{align*}
By Corollary \ref{corollary:Lipschitz}, we have
\begin{align*}
\bP\left(L_f \geq \frac{\epsilon}{2r}\right) \leq \frac{4\mb_1 r}{\epsilon}.
\end{align*}
By Theorem \ref{theorem:convergence-pointwise}, we have
\begin{align*}
\bP\left(||f(x_j)||_{\HS} \geq \frac{\epsilon}{2}\right)  \leq 2 \exp\left(-\frac{D\epsilon}{4M}\log\left[1+\frac{M\epsilon}{2\sigma^2(\tilde{\mu}(\omega))}\right]\right).
\end{align*}
Putting everything together, we obtain
\begin{align*}
\bP\left(\sup_{x \in \Omega}||f(x)||_{\HS} \geq \epsilon\right)  
& \leq \frac{4\mb_1 r}{\epsilon} + 2\Ncal\exp\left(-\frac{D\epsilon}{4M}\log\left[1+\frac{M\epsilon}{2\sigma^2(\tilde{\mu}(\omega))}\right]\right)
\\
&\leq \frac{4\mb_1 r}{\epsilon} + 2\exp\left(-\frac{D\epsilon}{4M}\log\left[1+\frac{M\epsilon}{2\sigma^2(\tilde{\mu}(\omega))}\right]\right)\left(\frac{2\diam(\Omega)}{r}\right)^n
\\
&= ar + \frac{b}{r^n},
\end{align*}
where $a = \frac{4\mb_1}{\epsilon}$ and $b =2\exp\left(-\frac{D\epsilon}{4M}\log\left[1+\frac{M\epsilon}{2\sigma^2(\tilde{\mu}(\omega))}\right]\right)\left({2\diam(\Omega)}\right)^n$. 

Let us find the value $r > 0$ that minimizes the right hand side in the above expression. The function $g(r) = ar +\frac{b}{r^n}$ for $a, b > 0$ achieves its minimum on $(0, \infty)$ at $r = \left(\frac{bn}{a}\right)^{\frac{1}{n+1}}$, with the minimum value given by
\begin{align*}
g_{\min} = a^{\frac{n}{n+1}}b^{\frac{1}{n+1}}\left(n^{\frac{1}{n+1}} + n^{-\frac{n}{n+1}}\right).
\end{align*}

Substituting the value for $a$ and $b$, we obtain
\begin{align*}
\bP\left(\sup_{x \in \Omega}||f(x)||_{\HS} \geq \epsilon\right)  \leq a(n)\left(\frac{\mb_1\diam(\Omega)}{\epsilon}\right)^{\frac{n}{n+1}}
\exp\left(-\frac{D\epsilon}{4(n+1)M}\log\left[1+\frac{M\epsilon}{2\sigma^2(\tilde{\mu}(\omega))}\right]\right)
\end{align*}
where $a(n) =  2^{\frac{3n+1}{n+1}}\left(n^{\frac{1}{n+1}} + n^{-\frac{n}{n+1}}\right)$. This completes the proof of the theorem.
\end{proof}

\begin{proof}{\textbf{of Theorem \ref{theorem:representer}}}
It is straightforward to show that the optimization problem (\ref{equation:general}) has a unique solution $f_{\z,\gamma}$, which has the form $f_{\z,\gamma}(x) 
= \sum_{i=1}^{u+l}K(x,x_i)a_i$ for some $a_i \in \W$. Under the feature map representation $\Phi_K$, we have
$$
f_{\z,\gamma}(x) = \sum_{i=1}^{u+l}K(x,x_i)a_i = \sum_{i=1}^{u+l}\Phi_K(x)^{*}\Phi_K(x_i)a_i = \Phi_K(x)^{*}\h,
$$
where
$$
\h = \sum_{i=1}^{u+l}\Phi_K(x_i)a_i = \Phi_K(\x)\a,
$$
as we claimed.
\end{proof}

To prove Theorem \ref{theorem:leastsquare-K}, we first consider the following operators.

The sampling operator $S_{\mathbf{x}}: \H_K \rightarrow {\W}^l$ is defined by $S_{\mathbf{x}}(f) = (f(x_i))_{i=1}^l$, 
for any $\y = (y_i)_{i=1}^l \in \mathcal{W}^l$,
\begin{align*}
\langle S_{\mathbf{x}}f, \mathbf{y} \rangle_{{\W}^l}&= \sum_{i=1}^l\langle f(x_i), y_i\rangle_{\W} = \sum_{i=1}^l\langle  K^{*}_{x_i}f,y_i\rangle_{\H_K}\\
&= \sum_{i=1}^l\langle f, K_{x_i}y_i\rangle_{\H_K} = \langle f, \sum_{i=1}^l K_{x_i}y_i\rangle_{\H_K}.
\end{align*}
Thus the adjoint operator $S_{\mathbf{x}}^{*}: {\W}^l \rightarrow \H_K$ is given by
\begin{equation}
S_{\mathbf{x}}^{*}\mathbf{y} = S_{\mathbf{x}}^{*}(y_1, \ldots, y_l) = \displaystyle{\sum_{i=1}^lK_{x_i}y_i}, \;\;\; \y \in \mathcal{W}^l,
\end{equation}
and the operator $S_{\mathbf{x}}^{*}S_{\mathbf{x}}: \H_K \rightarrow \H_K$ is given by
\begin{equation}
S_{\mathbf{x}}^{*}S_{\mathbf{x}}f = \displaystyle{\sum_{i=1}^lK_{x_i}f(x_i)} = \sum_{i=1}^l K_{x_i}K^{*}_{x_i}f.
\end{equation}

Consider the 
operator $E_{C, \x}: \H_K \mapto \mathcal{Y}^{l}$, defined by
\begin{equation}
E_{C,\x}f = (CK_{x_1}^{*}f, \ldots, CK_{x_l}^{*}f),
\end{equation}
with $CK_{x_i}^{*}: \H_K \mapto \mathcal{Y}$ and $K_{x_i}C^{*}: \mathcal{Y} \mapto \H_K$.
For $\mathbf{b} = (b_1, \ldots, b_l) \in \mathcal{Y}^{l}$, we have
\begin{eqnarray}
\la \mathbf{b}, E_{C,\x}f\ra_{\mathcal{Y}^{l}} = \sum_{i=1}^l  \la b_i, CK_{x_i}^{*}f\ra_{\mathcal{Y}}
 = \sum_{i=1}^l \la K_{x_i}C^{*}b_i, f\ra_{\H_K}.
\end{eqnarray}
The adjoint operator $E_{C,\x}^{*}:\mathcal{Y}^{l} \mapto \H_K$ is thus
\begin{equation}
E_{C,\x}^{*}: (b_1, \ldots, b_l) \mapto \sum_{i=1}^lK_{x_i}C^{*}b_i.
\end{equation}
The operator $E_{C,\x}^{*}E_{C, \x}:\H_K \mapto \H_K$ is then
\begin{equation}
E_{C,\x}^{*}E_{C, \x}f \mapto \sum_{i=1}^lK_{x_i}C^{*}CK_{x_i}^{*}f,
\end{equation}
with $C^{*}C: \mathcal{W} \mapto \mathcal{W}$.

\begin{proof}{\textbf{of Theorem \ref{theorem:leastsquare-K}}}
Since $f(x) = K_x^{*}$, we have
\begin{eqnarray}\label{equation:vector-lsq2}
f_{\z, \gamma} = \argmin_{f \in \H_K} \frac{1}{l}\sum_{i=1}^l||y_i - CK_{x_i}^{*}f||^2_{\mathcal{Y}}
+ \gamma_A||f||^2_{\H_K} + \gamma_I \la \f, M\f\ra_{\mathcal{W}^{u+l}}.
\end{eqnarray}
Using the operator $E_{C, \x}$, this becomes
\begin{equation}
\label{equation:multiview-lsq2}
f_{\z, \gamma} = \argmin_{f \in \H_K} \frac{1}{l}||E_{C,\x}f-\y||^2_{\mathcal{Y}^l}
 + \gamma_A||f||^2_{\H_K} + \gamma_I \la \f, M\f\ra_{\mathcal{W}^{u+l}}.
\end{equation}
Differentiating (\ref{equation:multiview-lsq2}) and setting the derivative to zero
gives
\begin{equation}\label{equation:fz1}
(E_{C,\x}^{*}E_{C,\x} + l \gamma_A I + l\gamma_I S_{\x,u+l}^{*}MS_{\x,u+l})f_{\z, \gamma} = E_{C,\x}^{*}\y,
\end{equation}
which is
$$
f_{\z,\gamma} = (E_{C,\x}^{*}E_{C,\x} + l \gamma_A I_{\H_K} + l\gamma_I S_{\x,u+l}^{*}MS_{\x,u+l})^{-1}E_{C,\x}^{*}\y.
$$
On the set $\x = (xi)_{i=1}^{u+l}$, the operators $S_{\x,u+l}: \H_K \mapto \W^{u+l}$ and $S_{\x, u+l}^{*}: \W^{u+l} \mapto \H_K$ are given by
$$
S_{\x, u+l}f = (K_{x_i}^{*}f)_{i=1}^{u+l}, \;\;\; f \in \H_K,
$$
$$
S_{\x, u+l}^{*}\b  = \sum_{i=1}^{u+l}K_{x_i}b_i, \;\;\; \b \in \W^{u+l}.
$$
By definition of the operators $S_{\x,u+l}$ and $S_{\x,u+l}^{*}$, we have
$$
S_{\x, u+l}^{*}(I_{(u+l) \times l} \otimes C^{*})\y = \sum_{i=1}^lK_{x_i}(C^{*}y_i).
$$
Thus the operator $E_{C,\x}^{*}: \Y^l \mapto \H_K$ is
\begin{equation}
E_{C, \x}^{*} = S_{\x, u+l}^{*}(I_{(u+l) \times l} \otimes C^{*}).
\end{equation}
The operator  $E_{C,\x}^{*}E_{C,\x}: \H_K \mapto \H_K$  is given by
\begin{equation}
E_{C,\x}^{*}E_{C,\x} = S_{\x, u+l}^{*}(J^{u+l}_l \otimes C^{*}C)S_{\x, u+l}: \H_K \mapto \H_K,
\end{equation}
Equation (\ref{equation:fz1}) becomes
\begin{equation}\label{equation:fz2}
\left[S_{\x, u+l}^{*}(J^{u+l}_l \otimes C^{*}C + l \gamma_I M)S_{\x, u+l} + l \gamma_A I_{\H_K}\right]f_{\z,\gamma} = S_{\x, u+l}^{*}(I_{(u+l) \times l} \otimes C^{*})\y,
\end{equation}
which gives
\begin{equation}
f_{\z,\gamma} = \left[S_{\x, u+l}^{*}(J^{u+l}_l \otimes C^{*}C + l \gamma_I M)S_{\x, u+l} + l \gamma_A I_{\H_K}\right]^{-1}S_{\x, u+l}^{*}(I_{(u+l) \times l} \otimes C^{*})\y.
\end{equation}

For any $x \in \X$,
$$
(S_{\x, u+l}^{*}(I_{(u+l) \times l} \otimes C^{*})\y)(x) = \sum_{i=1}^lK(x,x_i)(C^{*}y_i) \in \W.
$$
Using the feature map $\Phi_K$, we have for any $x \in \X$,
$$
(S_{\x, u+l}^{*}(I_{(u+l) \times l} \otimes C^{*})\y)(x) = \sum_{i=1}^l\Phi_K(x)^{*}\Phi_K(x_i)(C^{*}y_i) \in \W,
$$
and for any $w \in \W$,
\begin{align*}
\la (S_{\x, u+l}^{*}(I_{(u+l) \times l} \otimes C^{*})\y)(x),w\ra_{\W} &= \sum_{i=1}^l\la \Phi_K(x)^{*}\Phi_K(x_i)(C^{*}y_i), w\ra_{\W} \nonumber\\
&= \sum_{i=1}^l \la \Phi_K(x_i)(C^{*}y_i), \Phi_K(x)w\ra_{\F_K} \nonumber\\
&= \la \Phi_K(\x)(I_{(u+l) \times l} \otimes C^{*})\y, \Phi_K(x)w\ra_{\F_K}.
\end{align*}
For any $f \in \H_K$,
$$
S_{\x,u+l}^{*}MS_{\x,u+l}f = \sum_{i=1}^{u+l}K_{x_i}(M\f)_i.
$$
For any $x \in \X$,
$$
(S_{\x,u+l}^{*}MS_{\x,u+l}f)(x) = \sum_{i=1}^{u+l}K(x,x_i)(M\f)_i = \sum_{i=1}^{u+l}\Phi_K(x)^{*}\Phi_K(x_i)(M\f)_i \in \W,
$$
and for any $w \in \W$,
$$
\la (S_{\x,u+l}^{*}MS_{\x,u+l}f)(x), w\ra_{\W} = \sum_{i=1}^{u+l}\la \Phi_K(x_i)(M\f)_i, \Phi_K(x)w\ra_{\F_K}.
$$
We have
$$
\sum_{i=1}^{u+l}\Phi_K(x_i)(M\f)_i = \sum_{i=1}^{u+l}\Phi_K(x_i)(M\Phi_K(\x)^{*}\h)_i = \Phi_K(\x)M\Phi_K(\x)^{*}\h \in \F_K.
$$
It follows that
\begin{equation}
\la (S_{\x,u+l}^{*}MS_{\x,u+l}f)(x), w\ra_{\W} = \la \Phi_K(\x)M\Phi_K(\x)^{*}\h , \Phi_K(x)w\ra_{\F_K}.
\end{equation}
Similarly, for any $f \in \H_K$,
\begin{align*}
S_{\x,u+l}^{*}(J^{u+l}_l \otimes C^{*}C)S_{\x,u+l}f &= S_{\x, u+l}^{*}(J^{u+l}_l \otimes C^{*}C)\f = \sum_{i=1}^{u+l}K_{x_i}((J^{u+l}_l \otimes C^{*}C)\f)_i
\\
&= \sum_{i=1}^{u+l}K_{x_i}((J^{u+l}_l \otimes C^{*}C)\Phi_K(\x)^{*}\h)_i.
\end{align*}
For any $x \in \X$,
\begin{align*}
(S_{\x,u+l}^{*}(J^{u+l}_l \otimes C^{*}C)S_{\x,u+l}f)(x) &= \sum_{i=1}^{u+l}K(x,x_i)((J^{u+l}_l \otimes C^{*}C)\Phi_K(\x)^{*}\h)_i
\\
&= \sum_{i=1}^{u+l}\Phi_K(x)^{*}\Phi_K(x_i)((J^{u+l}_l \otimes C^{*}C)\Phi_K(\x)^{*}\h)_i.
\end{align*}
For any $w \in \W$,
\begin{align*}
\la (S_{\x,u+l}^{*}(J^{u+l}_l \otimes C^{*}C)S_{\x,u+l}f)(x), w\ra_{\W} &= \la \sum_{i=1}^{u+l}\Phi_K(x_i)((J^{u+l}_l \otimes C^{*}C)\Phi_K(\x)^{*}\h)_i, \Phi_K(x)w\ra_{\W}
\\
&= \la \Phi_K(\x)(J^{u+l}_l \otimes C^{*}C)\Phi_K(\x)^{*}\h, \Phi_K(x)w\ra_{\F_K}.
\end{align*}
Equation (\ref{equation:fz2}) is then equivalent to
\begin{align*}
&\la \Phi_K(\x)(J^{u+l}_l \otimes C^{*}C)\Phi_K(\x)^{*}\h, \Phi_K(x)w\ra_{\F_K} + l \gamma_I \la \Phi_K(\x)M\Phi_K(\x)^{*}\h , \Phi_K(x)w\ra_{\F_K}
\\
&+ l \gamma_A \la \h, \Phi_K(x)w\ra_{\F_K} = \la \Phi_K(\x)(I_{(u+l) \times l} \otimes C^{*})\y, \Phi_K(x)w\ra_{\F_K}.
\end{align*}
for all $x \in \X$, $w \in \W$, which is
\begin{eqnarray*}
\la \Phi_K(\x)[(J^{u+l}_l \otimes C^{*}C) + l\gamma_I M]\Phi_K(\x)^{*}\h, \Phi_K(x)w\ra_{\F_K} + l \gamma_A \la \h, \Phi_K(x)w\ra_{\F_K}
\\
= \la \Phi_K(\x)(I_{(u+l) \times l} \otimes C^{*})\y, \Phi_K(x)w\ra_{\F_K}.
\end{eqnarray*}
for all $x \in \X$, $w \in \W$. This is satisfied if
$$
\left(\Phi_K(\x)[(J^{u+l}_l \otimes C^{*}C) + l\gamma_I M]\Phi_K(\x)^{*} + l \gamma_A I_{\F_K}\right)\h = \Phi_K(\x)(I_{(u+l) \times l} \otimes C^{*})\y.
$$
This completes the proof of the theorem.
\end{proof}

\bibliography{cite_RKHS}

\end{document}

%% file: style.tex
\newcommand{\inner}[2]{\ensuremath{\langle{#1},{#2}\rangle}}
\newcommand{\lmp}[2]{\ensuremath{\ell_{#2}^{#1}}}
\newcommand{\mapto}{\ensuremath{\rightarrow}}
\newcommand{\nmapto}{\ensuremath{\nrightarrow}}
\newcommand{\approach}{\ensuremath{\rightarrow}}
\newcommand{\imply}{\ensuremath{\Rightarrow}}
\newcommand{\inject}{\ensuremath{\hookrightarrow}}
\newcommand{\equivalent}{\ensuremath{\Longleftrightarrow}}
\newcommand{\inclusion}{\ensuremath{\hookrightarrow}}

\newcommand{\ol}[1]{\ensuremath{\overline{#1}}}

\newtheorem{observation}{Observation}
\newtheorem{hypothesis}{Hypothesis}
\newtheorem{notation}{Notation}

\newcommand{\A}{\mathcal{A}}
\newcommand{\B}{\mathcal{B}}
\newcommand{\X}{\mathcal{X}}
\newcommand{\Y}{\mathcal{Y}}
\newcommand{\Z}{\mathbb{Z}}
\newcommand{\N}{\mathbb{N}}
\newcommand{\R}{\mathbb{R}}

\newcommand{\G}{\mathbf{G}}
\newcommand{\D}{\mathbf{D}}
\newcommand{\bfC}{\mathbf{C}}
\newcommand{\bfW}{\mathbf{W}}
\newcommand{\bfS}{\mathbf{S}}
\newcommand{\bfD}{\mathbf{D}}
\newcommand{\bfH}{\mathbf{H}}
\newcommand{\bfA}{\mathbf{A}}
\newcommand{\bfG}{\mathbf{G}}
\newcommand{\bfB}{\mathbf{B}}
\newcommand{\bfJ}{\mathbf{J}}

\newcommand{\bbC}{\mathbb{C}}

\newcommand{\bP}{\mathbb{P}}
\newcommand{\bE}{\mathbb{E}}
\newcommand{\bC}{\mathbb{C}}

\newcommand{\Ecal}{\mathcal{E}}
\newcommand{\Fcal}{\mathcal{F}}

\newcommand{\C}{\mathcal{C}}
\newcommand{\Dcal}{\mathcal{D}}

\newcommand{\gfrak}{\mathfrak{g}}

\newcommand{\T}{\mathcal{T}}
\renewcommand{\O}{\mathcal{O}}

\newcommand{\Q}{\mathcal{Q}}

\newcommand{\M}{\mathcal{M}}
\renewcommand{\P}{\mathcal{P}}
\renewcommand{\S}{\mathcal{S}}

\newcommand{\la}{\langle}
\newcommand{\ra}{\rangle}
\newcommand{\F}{\mathcal{F}}
\renewcommand{\H}{\mathcal{H}}
\renewcommand{\L}{\mathcal{L}}

\newcommand{\Mcal}{\mathcal{M}}

\newcommand{\Fre}{Fr\'echet \;}
\newcommand{\Ga}{G\^ateaux \;}

\newcommand\conv{{\rm conv}}

\def\diag{{\rm diag}}
\def\diam{{\rm diam}}
\def\rank{{\rm rank}}
\def\cond{{\rm cond}}

\def\dist{{\rm dist}}
\def\best{{\rm best}}

\def\supp{{\rm supp}}
\def\sinc{{\rm sinc}}

\def\argmin{{\rm argmin}}

\def\argmax{{\rm argmax}}
\def \sgn{{\rm sgn}}

\def \range{{\rm range}}
\def \reg{{\rm reg}}

\def\Im{{\rm Im}}

\def\proj{{\rm proj}}
\def\cov{{\rm cov}}

\def\trace{{\rm tr}}
\def\loc{{\rm loc}}
\def\vec{{\rm vec}}
\def\nullspace{{\rm nullspace}}
\def\colspace{{\rm colspace}}
\def\rowspace{{\rm rowspace}}

\def\rvec{\widetilde{\rm vec}}

\def\curl{{\rm curl}}
\def\div{{\rm div}}
\def\har{{\rm har}}
\def\corr{{\rm corr}}

\def\sym{{\rm sym}}
\def\rw{{\rm rw}}
\def\target{{\rm target}}
\def\pos{{\rm pos}}
\def\neg{{\rm neg}}

\def\HS{{\rm HS}}
\def\eHS{{\rm eHS}}

\def\Log{{\rm Log}}
\def\Exp{{\rm Exp}}

\def\opt{{\rm opt}}
\def\vec{{\rm vec}}

\def\Sym{{\rm Sym}}
\def\logHS{{\rm logHS}}
\def\elogHS{{\rm elogHS}}

\def\logE{{\rm logE}}
\def\eTr{{\rm eTr}}
\def\etr{{\rm tr_X}}
\def\tr{{\rm tr}}
\def\Tmath{{\rm T}}
\def\Tr{{\rm Tr}}

\def\length{{\rm length}}
\def\p{{\mathbf p}}

\def\Xfrak{{\mathfrak X}}

\def\grad{{\nabla}}

\def\const{{\rm constant}}

\def\sym{{\rm sym}}
\def\skew{{\rm skew}}

\newcommand{\U}{\mathcal{U}}

\newcommand{\h}{\mathbf{h}}
\newcommand{\x}{\mathbf{x}}
\newcommand{\s}{\mathbf{s}}
\newcommand{\w}{\mathbf{w}}
\newcommand{\z}{\mathbf{z}}
\renewcommand{\a}{\mathbf{a}}
\renewcommand{\c}{\mathbf{c}}
\renewcommand{\v}{\mathbf{v}}
\newcommand{\e}{\mathbf{e}}
\newcommand{\n}{\mathbf{n}}
\renewcommand{\b}{\mathbf{b}}

\newcommand{\y}{\mathbf{y}}
\newcommand{\f}{\mathbf{f}}

\newcommand{\W}{\mathcal{W}}

\newcommand{\bX}{\mathbf{X}}
\newcommand{\bW}{\mathbf{W}}
\newcommand{\bF}{\mathbf{F}}
\newcommand{\bb}{\mathbf{b}}

\newcommand{\bt}{\mathbf{t}}

\newcommand{\vecpi}{\overrightarrow{\pi}}
\newcommand{\veci}{\overrightarrow{i}}
\newcommand{\vecj}{\overrightarrow{j}}
\newcommand{\veck}{\overrightarrow{k}}

\def\1{\mathbf{1}}

\renewcommand{\u}{\mathbf{u}}

\newcommand{\gAI}{{\rm gAI}}

\newcommand{\Ncal}{\mathcal{N}}

\newcommand{\stein}{{\rm stein}}

\newcommand{\Pcal}{\mathcal{P}}
\newcommand{\Pb}{\mathbb{P}}

\newcommand{\logdet}{{\rm logdet}}

\newcommand{\jeff}{{\rm jeff}}

\newcommand{\detX}{{\rm det_X}}
\newcommand{\logdetX}{{\rm logdet_X}}

\newcommand{\Lcal}{\mathcal{L}}

\newcommand{\PT}{\mathbb{PT}}

\newcommand{\Bcal}{\mathcal{B}}

\newcommand{\mb}{\mathbf{m}}

\newcommand{\norm}{{\rm norm}}